\documentclass{llncs}
\usepackage[text={14cm,19.3cm},centering]{geometry}

\usepackage{multicol,aliascnt,remreset}
\usepackage{times}
\usepackage{helvet}
\usepackage{courier}
\usepackage{indentfirst,amssymb,amsmath,latexsym,graphicx}
\usepackage[utf8]{inputenc}
\usepackage[T2A,T1]{fontenc}
\usepackage[russian,english]{babel}
\newcommand{\commentout}[1]{}
\newcommand{\ttab}{\hspace*{.2in}}

\newcommand{\liff}{\leftrightarrow}

\newcommand{\nd}{\wedge}

\newcommand{\bt}{{\cal T}}

\newcommand{\Si}{\Sigma}
\newcommand{\si}{\sigma}
\newcommand{\D}{\Delta}
\newcommand{\DD}{{\cal{D}}}
\newcommand{\LL}{{\cal L}}
\newcommand{\F}{{\cal F}}
\newcommand{\A}{{\cal A}}
\newcommand{\Obj}{{\cal O}}
\newcommand{\Stat}{{\cal P}}

\def\M{{\mathcal{M}}}

\def\sig{{\tt sig}\,}
\def\forget{{\tt forget}\,}
\def\Cons{{\tt Cons}\,}
\newtheorem{teo}{Theorem}
\newtheorem{prop}[teo]{Proposition}
\newtheorem{cor}[teo]{Corollary}
\newtheorem{lem}[teo]{Lemma}
\numberwithin{teo}{section} 
\newtheorem{fact}[teo]{Fact}
\newtheorem{de}[teo]{Definition}

\begin{document}

\pdfinfo{
/Title (Progression of Decomposed Local-Effect Action Theories)
/Subject (Short version appeared in the Proceedings of the AAAI-2013 conference)
/Author (Denis Ponomaryov http://persons.iis.nsk.su/en/person/ponom and Mikhail Soutchanski http://www.scs.ryerson.ca/mes/)
}


\title{Progression of Decomposed Local-Effect Action Theories}
\author{Denis Ponomaryov\inst{1}, Mikhail Soutchanski\inst{2}}
\institute{Institute of Artificial Intelligence, University of Ulm, Germany\\ A.P. Ershov Institute of Informatics Systems, Novosibirsk State University, Russia \and Department of Computer Science, Ryerson University, Toronto, Canada\\ \email{ponom@iis.nsk.su, mes@scs.ryerson.ca}
}
\date{}


%
%



\maketitle

\begin{abstract}
In many tasks related to reasoning about consequences of a logical
theory, it is desirable to decompose the theory into a number of
weakly-related or independent components. However, a theory may
represent knowledge that is subject to change, as a result of executing
actions that have effects on some of the initial properties mentioned in the
theory. Having once computed a decomposition of a theory, it is advantageous to know whether a decomposition has to be computed
again in the newly-changed theory (obtained from taking into account changes
resulting from execution of an action). In the paper, we address
this problem in the scope of the situation calculus, where a
change of an initial theory is related to the notion of
progression. Progression provides a form of forward reasoning; it relies on forgetting values of those properties, which are subject to change, and computing new values for them. We consider decomposability and inseparability, two component properties known from the literature, and contribute by 1) studying the conditions when these properties are preserved and 2) when they are lost wrt progression and the related operation of forgetting. To show the latter, we demonstrate the boundaries using a number of negative examples. To show the former, we identify cases when these properties are preserved under forgetting and progression of initial theories in local-effect basic action theories of the situation calculus. Our paper contributes to bridging two different communities in Knowledge Representation, namely research on modularity and research on reasoning about actions.
\end{abstract}

\section{Introduction}
Modularity of theories has been established as an important research topic
in knowledge representation. It includes both theoretical and
practical aspects of modularity of theories formulated in different logical
languages  ($\LL$), ranging from weak (but practical) description logics
(DLs) such as $\cal EL$ and DL-Lite to more expressive logics 
\cite{ModularTheoryAndPractice,GruningerApplOntol2012,DecompOnt,FormalPropMod,DecompLogics,Vescovo}, to cite a few.
Surprisingly, this research topic is little explored in the context of
reasoning about actions. 
More specifically, it is natural to decompose a large heterogeneous theory
covering several loosely-coupled application domains into components
that have little or no intersection in terms of signatures. Potentially,
such decomposition can facilitate solving the projection problem, which requires
answering whether a given logical formula is true after executing a sequence
of actions (events). In cases, when a query is a logical formula composed from
symbols occurring in only one of the components, the query can be answered more
easily than in the case when the whole theory is required. This decomposition can
help in solving other reasoning problems (e.g. planning or high-level
program execution) that require a solution to the projection problem as a
prerequisite. To the best of our knowledge, the only previous work that explored
decomposition of logical theories for the purposes of solving the projection problem are the papers 
\cite{EyalAmirAAAI2000,EyalAmirKR2002}. These papers investigate decomposition
in the situation calculus \cite{MH69,Reiter2001}, a well-known logical formalism for
representation of actions and their effects. The author proposed reasoning
procedures for a situation calculus theory by dividing the whole theory syntactically into weakly-related partitions. Specifically, he developed algorithms that use local computation inside syntactically-identified partitions and message passing between partitions. 
We take a different approach in our paper: instead of decomposing the whole action theory into subsets, as in \cite{EyalAmirAAAI2000,EyalAmirKR2002}, 
we consider signature decompositions of an initial theory only.
Our components are not necessarily syntactic subsets of the initial theory. 
We concentrate on foundations, and explore properties of components produced 
by our decomposition. Whenever possible, we try to formulate these properties 
in a general logical language $\LL$, that is a fragment of second order logic; however, when necessary, we focus on a specific logic.

This paper considers the decomposability and inseparability properties of 
logical theories. These properties are well known in research on
modularization in the area of knowledge representation
\cite{DecompOnt,DecompLogics,FormalPropMod,ConsExtEL}, but have not been studied previously in the scope of the situation calculus. Both
properties are concerned with subdividing theories into components
to facilitate reasoning. Informally, decomposability of a theory
means that the theory can be equivalently represented as a union of two
(or several) theories sharing a given set ($\D$) of 
signature symbols. Inseparability of theories wrt some
signature $\D$ means that the theories have the same set of
logical consequences in the signature $\D$. If a theory ($\bt$) is
$\D$--decomposable into $\D$--inseparable components, then (under
certain restrictions on the underlying logic) each component of the
decomposition contains all information from $\bt$ in 
its own signature. This is an ideal case of
decomposition, since in this case the problem of entailment from
$\bt$ can be reduced to entailment from components, which are
potentially smaller than the theory $\bt$.

In the area of reasoning about actions, an initial logical theory
represents knowledge that is subject to change due to the effects of
actions on some of the properties mentioned in the theory. It can
be updated with new information caused by actions, while some other knowledge should
be forgotten, as it is no longer true in the next situation. We consider
two types of update operators: 1) forgetting in arbitrary theories and
2) progression of theories in the situation calculus. Forgetting is a
well-known operation on theories first introduced by Fangzhen Lin
and Ray Reiter in their seminal paper \cite{ForgetIt}. Forgetting
a signature $\si$ in a theory $\bt$ means obtaining a theory
indistinguishable from $\bt$ in the rest of the signature symbols
$\sig(\bt)\setminus\si$. In this sense, forgetting a signature is
close to the well-known notion of uniform interpolation.
Forgetting a ground atom $P(\bar{t})$ in a theory $\bt$ results in a
theory that implies all the consequences of $\bt$ ``modulo'' the
truth value of $P(\bar{t})$. The operation of forgetting is
closely related to progression in basic action theories in the
situation calculus.

The situation calculus \cite{Reiter2001} is a knowledge representation logical
formalism, which has been designed for axiomatization of problems in 
planning and high-level program execution. 
The idea is to 1) axiomatize a set of initial states (as an initial theory), 
2) axiomatize preconditions telling when actions can be performed, and then 3)
add the axioms about the effects of actions on situation-dependent properties. 
After these steps, one can reason about the consequences of sequences of actions to determine whether 
properties of interest hold in a given situation resulting from executing 
a sequence of actions  and whether a certain sequence of actions is consecutively executable. 
In the situation calculus, the so-called basic action theories represent such 
axiomatizations  \cite{Reiter2001}. Each basic action theory
contains an initial theory that represents incomplete knowledge about an 
initial situation $S_0$. In a special case, when there is complete knowledge about
a finite number of individuals having unique names, the initial theory
can be implemented as a relational database \cite{Reiter2001,DeGiacomoMancini2004}.
 Roughly, a basic action theory $\DD$ is a union of an
initial theory $\DD_{S_0}$ with some theory $\bt$, defining
transitions among situations, and a set of ``canonical''
axioms assumed to be true for all application problems represented in the
situation calculus. Informally speaking, an update of the initial theory after 
execution of an action is called ``progression of the initial theory wrt an action''.
More precisely, progression of $\DD_{S_0}$ wrt some
action $\alpha$ is a logical consequence of $\DD$ which contains
all information from $\DD$ about the situation resulting from the
execution of $\alpha$ in the situation $S_0$. Ideally, it is computed as updating
$\DD_{S_0}$ with some logical consequences of $\bt$, once all information in 
$\DD_{S_0}$, which is no longer true in the resulting situation, has been forgotten. 
Progression is
important for practical agents with indefinite horizon since progression is
the only feasible way of maintaining knowledge about the world. 
Exploiting modularity in the vast agent's knowledge is important to guarantee 
that progression of the agent's knowledge will be computationally feasible.

   Historically, the ``situation calculus''  (earlier referred to as ``situational logic'')
is the earliest logical framework developed  in the area of artificial intelligence (AI).
Having been developed in the 1960s by John McCarthy and his colleagues \cite{McC63,MH69,Green69}, 
it is one of the most popular logical frameworks for reasoning about actions;
it is presented in most well-known textbooks on AI. 
It is worth mentioning that
there are both conceptual and technical differences between the
situation calculus, designed for reasoning about
arbitrary actions, and the Floyd--Hoare logic, Dijkstra's
predicate transformers, and dynamic logic, and other related
formalisms, which have been developed for reasoning about the
correctness of computer programs. For example, the latter category of
formalisms would consider the operator assigning a new value to a
variable in a program as a primitive action, while the former
would consider as primitive the actions on higher level of
abstraction, such as moving a book from its current
location to the table. For this reason, the situation calculus is
chosen as foundation for high-level programming languages in
cognitive robotics \cite{CogRobKRHandbook}. In our paper, when we
refer to the ``situation calculus'', we are following the axiomatic
approach and notation proposed by R.Reiter \cite{Reiter2001} who
developed a general approach to axiomatizing direct effects and
non-effects of actions. It has been observed for a long time that
in practical applications, real-world actions have no effect on
most properties. However, it was Reiter who first proposed an
elegant axiomatization that represents compactly non-effects of
actions. Reiter's book covers several extensions of the situation
calculus to reasoning about concurrent actions, instantaneous
actions, processes extended in time, interaction between action
and knowledge, stochastic actions, as well as high-level
programming languages based on the situation calculus. In our
paper, we will focus on the cases of situation calculus when actions are sequential,
atemporal, and deterministic. Despite this focus, our
results can be subsequently adapted to characterize more general
classes of actions. The main limitation of our work is in
concentrating on direct effects only. Indirect effects of actions
are beyond the scope of the present study and will be considered in future work.

In this paper, we are interested in the case when the initial theory is decomposed into
inseparable components, studying which conditions guarantee preservation of 
decomposability and inseparability of components after forgetting or progression. We would like to avoid recomputing a decomposition of an updated
initial theory after executing an action. Moreover, we would
like to know whether the components remain inseparable after
progression. Such invariance of decomposability and inseparability wrt progression
is important since progression may continue indefinitely as long as new 
actions are being executed. If decomposability and inseparability are always preserved, 
then it would suffice to compute a
decomposition of the initial theory once -- this decomposition will remain ``stable"
after progression wrt any arbitrary sequence of actions. Additionally, if an executed action
has effects only on one component of the initial theory, then we would
like to be able to compute progression using only this part
instead of the whole initial theory. This leads to the question of
when the decomposability and inseparability properties are
preserved under progression and under forgetting. To answer this question we have
to better understand  the properties of these two operations. In our study, for brevity,
when we refer to ``decomposability'' and ``inseparability'' properties of components,
we will use the phrase {\it component properties}.

	This paper contributes to the general understanding of 
forgetting and progression in the literature, since new results on them are needed for 
the purposes of our investigation. Not surprisingly, both forgetting and progression 
have intricate interactions with properties of decomposed components. 
We will demonstrate that, in general, it is very difficult to guarantee the preservation of
decomposability and inseparability, because there is a certain
conceptual distance between these notions on one hand, and
forgetting and progression on the other -- we provide examples
witnessing this. Nevertheless, we will identify cases when these properties remain
invariant. Our results show that some of these cases have a
practically important formulation. An important contribution of the paper is in 
formulating clear negative examples that demonstrate cases when decomposability
and inseparability are lost under progression. Thus, the paper contributes to understanding the limits of the component approach based on these properties. In particular, our examples
demonstrate that there is little hope to preserve inseparability if the different components
share a fluent. Decomposability turns out to be also a fragile property that
can be easily lost after executing just one action in a simple basic action theory.
Overall, this paper contributes by advancing the study of forgetting and progression,
and also by carrying out a thorough and comprehensive study of when decomposability
and inseparability are preserved and when they are lost.

We start in Section 2 by introducing basic notations and then provide a survey on decomposability and inseparability, the two component properties of theories considered in this paper. Then in Section 3 we
introduce the basics of the situation calculus, proceeding to the
component properties of forgetting in Section
\ref{Section_Forgetting} and progression in Section
\ref{Section_Progression}. The last section, Section 5, includes a summary of
the obtained results. A preliminary shorter version of this paper (without proofs)
appeared in the proceedings of AAAI-13 conference \cite{PonomaryovSoutchanski2013}.
This extended version of our paper includes new results not mentioned
in the conference version as well as proofs\commentout{ In Section~\ref{Section_Forgetting},
the new results include propositions about cases when inseparability of
components is preserved under forgetting, and when it is lost. 
In Section~\ref{Section_Progression}, this version includes a new example for 
the case when decomposition components split, and when inseparability is lost 
after progression, as well as the new Theorem \ref{Teo_PreservationOfInseparability}. 
The Theorem \ref{Teo_PreserveCompLocalEffect} in this version 
was briefly mentioned in the conference paper as the Theorem 1, but without 
any discussion and without a proof.} and a detailed background
material and discussion of previously published results about forgetting and
progression, in order to make this paper self-contained.

\section{Background}\label{Sect_Background}
\subsection{Conventions and Notations}
Let $\LL$ be a logic (possibly many-sorted), which is a fragment (a set of sentences) of
second-order logic (either by syntax or by translation of
formulas), and has the standard model-theoretic Tarskian semantics.
We call the \textit{signature} a subset of non-logical symbols of
$\LL$ (and treat equality as a logical symbol). If $\M_1$ and $\M_2$ are two many--sorted structures and
$\D$ is a signature then we say that $\M_1$ and $\M_2$ \textit{agree} on
$\D$ if they have the same domains for each sort and the same
interpretation of every symbol from $\D$. If $\M$ is a structure
and $\si$ is a subset of predicate and function symbols from $\M$,
then we denote by $\M|_\si$ the \textit{reduct} of $\M$ to
$\si$, i.e., the structure with predicate and function names from
$\si$, where every symbol of $\si$ names the same entity as in
$\M$. The structure $\M$ is called \textit{expansion} of
$\M|_\si$. For a set of formulas $\bt$ in $\LL$, we denote by
$\sig(\bt)$ the signature of $\bt$, i.e. the set of all
non-logical symbols which occur in $\bt$. We will use the same
notation $\sig(\varphi)$ for the signature of a formula $\varphi$
in $\LL$. If $t$ is a term in the logic $\LL$
then the same notation $\sig(t)$ will be used for the set of all non-logical symbols occurring in $t$.
Throughout this paper, we use the notion of \textit{theory} as a
synonym for a set of formulas in $\LL$, which are sentences when
translated into second-order logic. Whenever we mention a set of
formulas, it is assumed that this set is in $\LL$, if the context
is not specified. For two theories, ${\bt}_1$ and ${\bt}_2$, the
notation ${\bt}_1\equiv{\bt}_2$ will be the abbreviation for
the semantic equivalence. If $\bt$ is a set
of formulas in $\LL$ and $\D$ is a signature, then $\Cons(\bt,\D)$
will denote the set of semantic consequences of $\bt$ (in $\LL$)
in the signature $\D$, i.e. $\Cons(\bt,\D)=\{\varphi\in\LL \ \mid \
\bt\models\varphi \ \text{and} \ \sig(\varphi)\subseteq\D\}$. We
emphasize that this is a notation for a set of formulas in $\LL$,
because $\bt$ may semantically entail formulas that are outside of $\LL$.

\subsection{Basic Facts about Decomposability and Inseparability}
In the area of theory modularization, a module (or component) is usually understood as a set of theory consequences that satisfy certain properties. The latter are determined by requirements to a module in the context of application. Some approaches follow the idea that a  module should be a syntactic subset of the axioms of a given theory. For instance, a theory can be partitioned into subsets of axioms meeting certain requirements of balance among the  partition. Thereafter, reasoning wrt the initial theory can be reduced to reasoning within the obtained components via a message passing algorithm, which communicates between the partitions to find information needed to answer a query \cite{AmirAIJ2005}. As a rule, the information to be communicated relates to the signatures shared between the partitions. The advantage of this approach is that the partitioning algorithm can be relatively simple and rely on syntactic analysis of theory axioms, thereby circumventing semantics. On the other hand, it may not be possible to eliminate some dependences between the partitions, if they are induced by syntactic form of the axioms. For instance, if a theory $\bt$ consists of the axioms $\{\forall x \ P(x), \ \ \forall x \  (P(x)\leftrightarrow Q(x))\}$, then it may not be possible to infer that it can be represented as the union of two components $\{\forall x P(x)\}$ and $\{\forall x Q(x)\}$, which do not share any signature symbols. In other words, a theory with syntactic dependencies may have an axiomatization that yields a partitioning into components, which either do not have symbols in common, or in a more  general case, share a fixed signature (given as a parameter of decomposition). 

In our paper, we adopt the following notion that was introduced in \cite{DecompLogics} and applied to the study of modularity in \cite{DecompOnt}.

\begin{de}[$\D$--decomposability property]\label{Def_Decomposable}
Let $\bt$ be a theory in $\LL$ and $\D\subseteq\sig(\bt)$ a
subsignature. We call $\bt$ \emph{$\D$--decomposable}, if there are
theories ${\bt}_1$ and ${\bt}_2$ in $\LL$ such that
\begin{itemize}
\item $\sig(\bt_1)\cap\sig(\bt_2)=\D$, but
$\sig(\bt_1)\neq\D\neq\sig(\bt_2)$;

\item $\sig(\bt_1)\cup\sig(\bt_2)=\sig(\bt)$;

\item ${\bt}\equiv{\bt}_1\cup{\bt}_2.$
\end{itemize}

The pair $\langle{\bt}_1, {\bt}_2\rangle$ is called
\emph{$\D$--decomposition} of $\bt$ and the theories ${\bt}_1$ and
${\bt}_2$ are called \emph{$\D$--decomposition components} of $\bt$. We
will sometimes omit the word ``decomposition" and call the sets
${\bt}_1$ and ${\bt}_2$ simply components of $\bt$, when the
signature $\D$ is clear from the context. The sets
$\sig(\bt_1)\setminus\D$ and $\sig(\bt_2)\setminus\D$ are called
signature ($\D$--decomposition) components of $\bt$.
\end{de}

The notion of $\D$--decomposition is defined using a pair of
theories, but it can be easily extended to the case of a family of
theories. It is important to realize that ${\bt}_1$ and ${\bt}_2$
are not necessarily subsets of axioms of $\bt$ in the above definition.
We only require that ${\bt}\equiv{\bt}_1\cup{\bt}_2$. Clearly, if
$\LL$ satisfies compactness and $\bt$ is a finite
$\D$--decomposable theory in $\LL$ for a signature $\D$, then
there is a $\D$--decomposition $\langle{\bt}_1, {\bt}_2\rangle$ of
$\bt$, where $\bt_1$ and $\bt_2$ are finite.

Note that the axioms of theory $\bt$ given before Definition \ref{Def_Decomposable} can not be (syntactically) partitioned into subsets having no signature symbols in common. However $\bt$ has a different axiomatization given by $\forall x P(x)$ and $\forall x Q(x)$ and hence, is $\varnothing$-decomposable. Thus, in general $\D$-decomposition can be finer than a syntactic partitioning based on a particular axiomatization of a theory. 

According to the definition of decomposability, computing a decomposition means finding another (equivalent) representation of a theory, which defines the required components. This means that a decomposition procedure must employ logical reasoning. Therefore, potentially it is more computationally complex than syntactic partitioning, which splits a theory into syntactic subsets of axioms. However, the research on algorithmic properties of decomposability (see e.g., \cite{ComplexityFOL,DisjointANDDecomposition,DecompOnt,DecompFOL,DecompLogics}) shows that deciding whether a theory is $\D$-decomposable turns out 
to be not harder than deciding the entailment in the underlying logic. Studying the complexity of decomposability in different logics is an ongoing research topic. 
An algorithm for computing decomposition components can be obtained, e.g., 
from a procedure of computing uniform interpolants (if the logic enjoys efficient uniform 
interpolation, see Proposition 2 in \cite{DecompLogics} and \cite{DecompFOL}), or 
by applying the technique of eliminating non-$\D$-symbols from the axioms of a theory. 
The technique is described in \cite{DecompOnt} for the logics $\cal{EL}$ and 
\textit{DL-Lite}, which is further studied in \cite{ConceptInterpolation} and can be extended to more expressive 
Description Logics. 
Using any form of equivalent rewriting means that the obtained axiomatization may be of a size larger than the set of axioms of the original theory. In particular, finding a decomposition may imply computing explicit definitions, in which case the size of the components depends on the complexity of such definitions in the underlying logic. For instance, a decomposition component may be of size exponentially larger than the original theory, which is evidenced by Example 28 in \cite{DecompOnt}. It is known that in general there is no upper bound on the complexity of explicit definitions in first-order logic \cite{Friedman,Mundici} and computing them is usually harder than entailment in FOL fragments (see e.g. \cite{BethDefinabilityDLs}). On the other hand, one can take control over the growth of the component sizes by carefully choosing which signature $\Delta$ can be shared between the components. Using Example 28 from \cite{DecompOnt}, for instance, it possible to describe a situation when tuning up $\Delta$ can exponentially reduce the component sizes. In general however, this question motivates research on the succinctness of explicit definitions and uniform interpolants in different logics.

\medskip

An important requirement often considered in the literature is that a module must contain all information about a signature of interest $\D$, which is typically a subset of the signature of the module. 
In other words, it is required that a module must entail the same consequences in signature $\D$, as the source theory. Having fixed a signature $\D$, the ability to see differences between two theories strongly depends on the logic being used as ``lens'' for their examination: the more expressive power the logic employed has, the more differences it is possible to see. Probably the most powerful tool in measuring similarity of theories is the language of second-order logic. If two theories have the same sets of second-order consequences in a signature $\D$, then the classes of reducts of their models onto $\D$ coincide, i.e. both theories ``define'' the same semantics for $\D$-symbols. Keeping in mind that a module is usually understood as a set of consequences of a source theory, it is important to note the following model-theoretic fact, which will be helpful for grasping the results of this paper. It says that if a logic $\LL$ is weaker than second-order, then in general, a set of consequences of a theory $\bt$ in $\LL$ may not be able to capture the intended semantics of symbols from a subsignature $\D$, as defined by $\bt$. 

\begin{fact}\label{Fact_ConsNoModelExpansion}
If $\bt$ is a theory in $\LL$ and $\D$ a signature, then some
models of $\Cons(\bt,\D)$ may not have an expansion to a model of
$\bt$.

Indeed, let $\LL$ be first-order logic and $\{P,f\}$ be a signature, where $P$ 
is a unary predicate and $f$ is a unary function. Let $\bt$ be a theory saying 
that $f$ is a bijection between the interpretation of $P$ and its complement. 
Thus, $\bt$ axiomatizes the class of models, where the interpretation of $P$ 
and its complement are of the same cardinality. Let $\M$ be a model from this class 
and let $\cal N$ be a model of the same signature $\{P\}$ in which the interpretation 
of $P$ is a countable set, but the complement is uncountable. The models $\M$ 
and $\cal N$ are elementary equivalent, i.e., no formula in signature $\{P\}$ can 
distinguish between these two models. For instance, this can be shown by using the fact  (e.g., see \cite{Ershov1980})
that every sentence in signature $\{P\}$ is equivalent to a boolean combination 
of formulas $\exists^ {\geqslant m} P$ and $\exists^ {\geqslant m} \neg P$,
where an integer $m > 0$, which mean ``$P$ (respectively, $\neg P$) holds on 
at least $m$ distinct elements''. Therefore, $\cal N$ is a model of 
$\Cons(\bt,\{P\})$, but clearly, it has no expansion to a model of $\bt$. 
\end{fact}


As will be noted in Section \ref{Section_Forgetting}, forgetting is an operation, which gives a set of (second-order) consequences axiomatizing the same class of models, as the original theory, modulo forgotten signature/ground atom.

\medskip

It is known that in general, a set of consequences of a theory may not be finitely axiomatizable in the logic, in which the theory is formulated. For instance, the following example is widely known in the literature on Description Logics (e.g., see Section 3.2 in \cite{MathLogicLifeScience}).  


\begin{fact}\label{Fact_ConsNotFinitelyAx}
If $\bt$ is a theory in $\LL$ and $\D$ a signature, then
$\Cons(\bt,\D)$ may not be finitely axiomatizable in $\LL$.

Let $\bt$ be the first-order theory axiomatized by the following two
axioms:

$\forall x\, [A(x)\rightarrow B(x)\, ]$

$\forall x\, [B(x)\rightarrow\exists y (R(x,y)\wedge B(y))\, ]$

\noindent 
Consider the signature $\D=\{A,R\}$. Then it is not hard to verify that $\Cons(\bt,\D)$ is equivalent to the following infinite set of formulas:

$\forall x\, A(x)\rightarrow\exists y R(x,y)$

$\forall x\, A(x)\rightarrow\exists y \exists u[\, R(x,y)\wedge R(y,u)\, ]$

$\forall x\, A(x)\rightarrow\exists y\exists u\exists v[\, R(x,y)\wedge R(y,u)\wedge R(u,v)\, ]$

...

\noindent 
By compactness, this theory is not finitely
axiomatizable in first-order logic.
\end{fact}


A well-known concept used to characterize similarity of two theories wrt a signature is \emph{inseparability}. This notion has also appeared in 
the context of entailment in Description Logics, e.g., see \cite{FormalPropMod,ConsExtEL}.

\begin{de}[$\D$--inseparability]\label{Def_Inseparable}
Theories $\bt_1$ and $\bt_2$ in $\LL$ are called \emph{$\D$--insepa\-rable}, for a
signature $\D$, if $\Cons(\bt_1,\D)=\Cons(\bt_2,\D)$. 
\end{de}

In other words, $\bt_1$ and $\bt_2$ are $\D$--inseparable if 
for any $\LL$-formula $\psi$ in signature $\D$, $\bt_1$ entails $\psi$ iff $\bt_2$ does.
That is, in the language $\LL$, no query in signature $\D$ separates $\bt_1$ and $\bt_2$ from each other. If $\Cons()$ is augmented with the third parameter specifying a logic, in which consequences are taken, then inseparability gives rise to a variety of notions of similarity between theories. As informally noted, two theories may be inseparable wrt a logic $\LL$, but entail different consequences wrt a language more expressive than $\LL$. For the purpose of this paper, we consider the non-parametrized notion of inseparability, assuming that the language of interest is the underlying logic $\LL$, in which the theories are formulated. This assumption is natural if one is only interested in entailment of $\LL$-formulas. 

\smallskip

Inseparability plays an important role for decompositions. Assume that we have a theory $\bt$ that is $\D$-decomposable into some components $\bt_1$ and $\bt_2$.  Although, the union $\bt_1\cup\bt_2$ must entail all
consequences of $\bt$ in the signature $\D$, the components $\bt_1$
and $\bt_2$ may not be $\D$--inseparable, if we demand them to be
finite. For example, the set of $\D$--consequences of $\bt_2$ may not be
finitely axiomatizable in $\LL$ by axioms of
$\bt_1$. This easily follows from Fact
\ref{Fact_ConsNotFinitelyAx} which shows that this phenomenon
is already possible in weak languages such as the sub-boolean
description logic $\mathcal{EL}$. On the other hand,
$\D$--inseparability of decomposition components can always be obtained if
the underlying logic $\LL$ has uniform interpolation (cf.
Proposition 2 in \cite{DecompLogics}). Both $\D$--decomposition and $\D$--inseparability are required to achieve modularity. Without $\D$--inseparability the components are not self-sufficient,
since a component may not entail some of the consequences in the shared vocabulary $\D$. The ideal case is when a theory $\bt$ has $\D$-decomposition into finite $\D$--inseparable components, as noted in Fact \ref{Fact_Reduc2Comp} further in this section.

In contrast to decomposability, deciding $\D$-inseparability of theories is usually 
harder, than deciding entailment in the logic in which the theories 
are formulated, as proved by the results in \cite{FormalPropMod}, \cite{ConsExtEL}, 
and the results on the complexity of deciding conservative extensions 
\cite{Lutz07conservativeextensions,Ghilardi06conservativeextensions}. 
However, there are practical cases in which this property is guaranteed to hold for 
any decomposition components of a given theory wrt a certain signature $\D$. 
For example, if the theory without equality is a set of ground atoms, then 
such theory is $\D$-decomposable iff there exist two subsets of atoms 
having only $\D$-symbols in common and containing at least one non-$\D$-symbol. 
This property is easy to check by computing syntactic connectedness of the signature
symbols. It is straightforward to verify that if $\D$ does not contain predicate symbols,
then the obtained decomposition components are guaranteed to be $\D$-inseparable. In other words, a set of ground atoms can be easily decomposed into inseparable components, if they share constants only with no common predicate symbols. A practically important generalization of theories consisting of ground atoms is \textit{proper$^+$} theories \cite{LakemeyerLevesque02,ProgressionLocalEffect}. Developing computationally tractable techniques for decomposition of \textit{proper$^+$} theories into inseparable components is of particular interest. We note the importance of having inseparable decomposition components below. 

\medskip

The well-known property of logics related to signature
decompositions of theories is the Parallel Interpolation Property (PIP)
first considered in a special form in \cite{Makinson} and studied
later in a more general form in \cite{DecompOnt}. 

\begin{de}[Parallel Interpolation Property]\label{Def_PIP}
A logic $\LL$ is said to have the \emph{parallel interpolation
property} (PIP) if for any theories $\bt_1$, $\bt_2$ in $\LL$ with
$\sig(\bt_1)\cap\sig(\bt_2)=\D$ and any formula $\varphi$ in
$\LL$, the condition $\bt_1\cup\bt_2\models\varphi$ yields the
existence of sets of formulas $\bt_1'$ and $\bt_2'$ in $\LL$ such
that:
\begin{itemize}
\item  $\bt_i\models\bt_i'$, for $i=1,2$, and
$\bt_1'\cup\bt_2'\models\varphi$;

\item $\sig(\bt_i')\setminus\D \subseteq
(\sig(\bt_i)\cap\sig(\varphi))\setminus\D$.
\end{itemize}
\end{de}

Note that PIP is closely related to Craig's interpolation
\cite{Craig1957Herbrand-Gentzen,Craig2008Synthese}. 
In fact, PIP  can be understood as an iterated version of Craig's interpolation 
in the logics that have compactness and deduction theorem (see Lemma~1 in 
\cite{DecompLogics}). Many logics known to have Craig interpolation -- e.g., second- 
and first-order logics, numerous modal logics, and some description logics, also have PIP. 
It is easy to note that, in the presence of PIP, decomposing a set $\bt$
of formulas into inseparable components wrt a signature $\D$ gives
a family of theories that imply all the consequences of $\bt$ in
their own subsignatures.

\begin{fact}\label{Fact_Reduc2Comp}
Let $\LL$ have PIP, $\bt$ be a theory in $\LL$, and $\D$ a
signature. Let $\langle{\bt}_1, {\bt}_2\rangle$ be a
$\D$--decomposition of $\bt$, with ${\bt}_1$ and ${\bt}_2$ being
$\D$--inseparable. Then for any formula $\varphi$ with
$\sig(\varphi)\subseteq\sig(\bt_i)$, for some $i=1,2$, we have
$\bt\models\varphi$ iff $\bt_i\models\varphi$.
\end{fact}

\begin{proof} 
Assume $\sig(\varphi)\subseteq\sig(\bt_1)$. If
$\bt_1\models\varphi$ then $\bt\models\varphi$ by definition of
$\D$--decomposability. If $\bt\models\varphi$ then ${\bt}_1\cup
{\bt}_2\models\varphi$ and by PIP, there are ${\bt}_1'$ and
${\bt}_2'$ such that $\bt_1\models\bt_1'$, $\bt_2\models\bt_2' \
$, ${\bt}_1'\cup {\bt}_2'\models\varphi$, and
$\sig({\bt}_2')\subseteq\D$. As ${\bt}_1$ and ${\bt}_2$ are
$\D$--inseparable, we obtain $\bt_1\models\bt_2'$ and conclude
that $\bt_1\models\varphi$. 
\end{proof}

\medskip

In other words, in the presence of PIP, inseparable decomposition components 
can be used instead of the original theory for checking the entailment of formulas
in the corresponding subsignatures. This is the reason for our
interest in the inseparability property in connection with decompositions.
As shown in \cite{EyalAmirAAAI2000,EyalAmirKR2002,AmirAIJ2005}, a decomposition of a theory 
can be beneficial even without inseparability thanks to applying the known methods 
of distributed reasoning via message passing between components. However, if
components are inseparable, then the reasoner can avoid message passing completely.

\subsection{Basics of the Situation Calculus}\label{Sect_SituationCalculus}
The language of the situation
calculus $\LL_{sc}$ has the first-order syntax over three sorts
\textit{action}, \textit{situation}, \textit{object}. It is
provided with the standard model-theoretic semantics. It is defined
over the countably infinite alphabet
$A_{sc}=\{do,\preceq,S_0,Poss\}\cup{\A}\cup{\F}\cup{\Obj}\cup\Stat$,
where $do$ is a binary function symbol of sort situation; $\preceq$ is a
binary relation on situations; $S_0$ is the constant of sort
situation; $Poss(a,s)$ is a binary predicate (saying whether $a$
is possible in $s$) with the first argument of sort action and the
second one of sort situation; $\A$ is a set of action functions
with arguments of sort object, $\F$ is a set of so-called fluents,
i.e., predicates having as arguments a tuple (vector) of sort
object and one last argument of sort situation; $\Obj$ is a set of
constants of sort object; and $\Stat$ is a set of static
predicates and functions, i.e., those that only have objects as
arguments. A symbol $v\in A_{sc}$ (predicate or function) is
called \textit{situation-independent} if $v\in
\A\cup{\Obj}\cup{\Stat}$. A ground term is of sort situation
iff it is either the constant $S_0$ or a term $do(A(\bar{t}),S)$, where
$A(\overline{t})$ is a ground action term and $S$ is a ground
situation term. For instance, a term $do(A_2(\bar{t_2}),
do(A_1(\bar{t_1}),S_0))$ denotes the situation resulting from executing
actions $A_1(\bar{t_1})$ and $A_2(\bar{t_2})$ consecutively from the
initial situation $S_0$. Informally, static predicates specify
object properties that never change no matter what actions are executed and fluents
describe those object properties that are situation--dependent.
The language of the situation calculus is used to formulate
\textit{basic action theories} ($\cal{BAT}$s); they
may serve as the formal specifications of planning problems. Every
$\cal{BAT}$ consists of a set of foundational axioms $\Si$, which
specify constraints on how the function $do$ and fluents must be
understood, a theory $D_{una}$ stating the unique name assumption
for action functions and objects, an initial theory $D_{S_0}$
describing knowledge about the initial situation $S_0$, a theory $D_{ap}$
specifying preconditions of action execution, and a theory $D_{ss}$
(the set of successor-state axioms, SSAs for short) which contains
definitions of fluents in the next situation in terms of static
predicates and the values of fluents in the previous situation.
A detailed example of a $\cal{BAT}$ is given at the end of this section.

\begin{example}[The Blocks World]\label{BWonly}
We illustrate some of the syntactic definitions using the well-known
Blocks World example. The domain of objects in this example consists of blocks 
that can form towers such that a block can be on the top of only one other block
and conversely only one block can be staying on the top of another block.
The unary predicate $Block$ holds for objects. The towers of blocks can be
described using the fluents $On(x,y,s)$, a block $x$ is on $y$ in situation $s$,
and $Clear(x,s)$, a block $x$ is clear in $s$  meaning that there is no block
on top of $x$ in situation $s$. The first fluent applies to pairs of blocks
in a tower, while the second fluent characterizes the top block.  An initial 
theory $D_{S_0}$ may include axioms about the initial configuration of blocks 
named using object constants $A,B,C$, e.g., $On(A,B,S_0)$, the block $A$ is on $B$ 
initially, $\neg\exists x On(x,A,S_0)$ and $\neg\exists x On(x,C,S_0)$, i.e., 
there are no blocks on top of blocks $A$ and $C$. Notice that both fluents are
predicates with situation as the last argument. A theory $D_{una}$ includes
axioms saying that all blocks $A,B,C$ are pairwise distinct.
The function $move(x,y,z)$ maps blocks $x,y,z$ into a separate sort action 
that represents moving block $x$ staying on top of block $y$ from block $y$ 
onto another block $z$. The precondition axioms $D_{ap}$ characterize when this
action is possible, e.g., $move(A,B,C)$ is possible in the initial situation $S_0$,
because both $A$ and $C$ are clear, but $move(B,A,C)$ is not possible in $S_0$, 
because the block $B$ is not clear, and it is not staying on $A$ in $S_0$. 
The situation $do(move(A,B,C),S_0)$ results from executing action $move(A,B,C)$
in the initial situation $S_0$. This action has effects on the fluents in the 
sense that the fluent predicates about $S_0$ may change their truth values in the situation
$do(move(A,B,C),S_0)$. Observe however that in $do(move(A,C,B),do(move(A,B,C),S_0))$
fluents are true iff they are true in $S_0$, since $move(A,C,B)$ is
inverse wrt $move(A,B,C)$ when these actions executed consecutively. The following 
successor state axiom characterizes all effects of all actions on the fluent $On$:
\medskip 

\noindent $
\begin{array}{l}
\hspace{0.1cm}
\forall x,y,z,a,s \ \ 
On(x,\!z,do(a,\!s)) \liff \ \exists y (a\!=\! move(x,\!y,\!z)) \lor On(x,\!z,\!s)\!\land\! 
     \neg \exists y ( a\! =\! move(x,\!z,\!y))
\end{array}
$
\\[1ex]
More specifically, block $x$ is on block $z$ after doing an action $a$ in situation $s$
iff the last action $a$ was moving $x$ from some other block $y$ to $z$, or if
$x$ was already on $z$ in $s$, and the last action $a$ did not move it elsewhere.
Subsequently, we do not write the $\forall$-quantifiers explicitly at front of
the axioms. 
\end{example}

In every basic action theory $\DD$ over a signature
$\si\subseteq A_{sc}$, the set of foundational axioms $\Si$ consists of the
following formulas \cite{Reiter1993} (note the axiom schema for induction):
\begin{trivlist}
\item $\forall \ a_1,a_2,s_1,s_2 \
[do(a_1,s_1)=do(a_2,s_2)\rightarrow a_1=a_2 \wedge s_1=s_2]$ \item
$\forall \ s \ \neg (s \preceq S_0 \wedge s\neq S_0)$ \item
$\forall \ s_1,s_2 \  [s_1\preceq s_2 \leftrightarrow \exists a \
(do(a,s_1)\preceq s_2) \vee s_1=s_2]$ \item $\forall P \
P(S_0)\wedge \forall a,s [P(s)\rightarrow P(do(a,s))] \rightarrow
\forall s P(s)$
\end{trivlist}
Reiter observed in \cite{Reiter1993} that foundational axioms $\Si$
generalize a single successor function over natural numbers to the case of
multiple successors over situations. The second order induction axiom serves 
to exclude non-standard trees as models.
\commentout{	
@article{Reiter1993,
  author    = {Raymond Reiter},
  title     = {Proving Properties of States in the Situation Calculus},
  journal   = {Artif. Intell.},
  volume    = {64},
  number    = {2},
  pages     = {337--351},
  year      = {1993},
  doi       = {10.1016/0004-3702(93)90109-O}
	}
}	

For every pair of distinct action functions $\{A,A'\}\subseteq\si$
and every pair $\langle a,b\rangle$ of distinct object constants
from $\si$, \, a theory $D_{una}$ contains axioms of
the form:
\begin{trivlist}
\item $a\neq b$
\item $\forall \ \bar{x},\bar{y} \ A(\bar{x})\neq A'(\bar{y})$
\item $\forall \ \bar{x},\bar{y} \ A(x_1,\ldots ,
x_n)=A(y_1,\ldots , y_n)\rightarrow x_1=y_1 \wedge\ldots\wedge
x_n=y_n$ if $A$ is n--ary.
\end{trivlist}
\noindent No other axioms are in $D_{una}$.

To define the remaining subtheories of $\cal{BAT}$, we need to
introduce the following syntactic notion (taken from \cite{PR99,Reiter2001}).
\begin{de}\label{De_UniformTheory}
A formula $\varphi$ in language $\LL_{sc}$ is called \emph{uniform} in a
situation term $S$ if:
\begin{itemize}
\item[1.] it does not contain quantifiers over variables of sort
situation;
\item[2.] it does not contain equalities between situation terms;
\item[3.] the predicates $Poss, \preceq$ do not occur in $\varphi$:
$\{Poss,\preceq\}\cap\sig(\varphi)=\varnothing$;
\item[4.] for every fluent $F\in\sig(\varphi)$, the term in the
situation argument of $F$ is $S$.
\end{itemize}
A set $\bt$ of formulas in $\LL_{sc}$ is called uniform in a
situation term $S$ if every formula of $\bt$ is uniform in $S$.
\end{de}
By definition, a set $\bt$ of formulas uniform in a situation term
$S$ either does not contain any situation terms (and hence, fluents), or 
the only situation term is $S$ which occurs as the situation argument of 
each fluent from $\sig(\bt)$. In the example above, the formula on 
the right hand side of the SSA is a formula uniform in $s$. If $\bt$ is a
set of sentences uniform in situation term $S$, i.e., $\bt$ has no
free variables, and $S$ occurs in formulas of $\bt$, then by items
(1), (2) of the definition, $S$ must be ground and thus, it must
either be the constant $S_0$, or have the form $do(A(\bar{t}),
S')$, where $S'$ is a ground situation term. Note that if the constant
$S_0$ or the binary function symbol $do$ is present in $\sig(\bt)$ and $\bt$ is
uniform in $S$, then necessarily $S_0\in\sig(S)$, or
$do\in\sig(S)$, respectively. By items (1) and (2), such theory $\bt$ does not
restrict the interpretation of the term $S$ and the cardinality of
the sort \textit{situation}, so the observations above lead to the
following property of uniform theories, which informally can be
summarized by saying that in sentences of a theory $\bt$ uniform
in a ground situation term $S$, we can understand this situation
term as playing a role of an index that can remain implicit.
Whenever we change the interpretation of $S$ (e.g., by choosing a
different interpretation for $do$ and $S_0$) in a model of $\bt$,
it suffices to ``move'' interpretations of fluents to this new
point to obtain again a model for $\bt$.

\begin{lem}\label{Lemma_ModelPropertyofUniformTheories}
Let $\bt$ be a set of sentences uniform in a ground situation term
$S$. Let $\M=\langle Act\cup Sit\cup Obj, \ \mathbf{do},
\mathbf{S_0}, \mathbf{F_1},\ldots , \mathbf{F_n}, \ \mathcal{I} \
\rangle$ be a model of $\bt$, where $Act$, $Sit$, and $Obj$ are
domains for the corresponding sorts \textit{action},
\textit{situation}, and \textit{object}, $\mathbf{do}$ and
$\mathbf{S_0}$ are the interpretations of the function $do$ and
constant $S_0$, respectively, $\mathbf{F_1},\ldots ,
\mathbf{F_n}$ are the interpretations of
fluents from $\sig(\bt)$, and $\mathcal{I}$ is the interpretation
of the rest of symbols from $\sig(\bt)$. For example,
$\mathbf{F_i}$ is a set of tuples $\langle u_1,\ldots
,u_{m-1},\mathbf{S}\rangle$, where $\mathbf{S}$ is the
interpretation of the ground term $S$ in $\M$.

Consider the structure $\M'=\langle Act\cup Sit'\cup Obj, \
\mathbf{do}', \mathbf{S_0}', \mathbf{F_1}',\ldots , \mathbf{F_n}',
\ \mathcal{I} \ \rangle,$  where $Sit'$ is an arbitrary set, the
domain for sort \textit{situation}, $\mathbf{do}'$ and
$\mathbf{S_0}'$ are arbitrary interpretations of $do$ and $S_0$ on
$Sit'$, respectively, and for $i\leqslant n$, $\mathbf{F_i}'$
denotes the interpretation of the fluent $F_i$ as a set of tuples
$\langle u_1,\ldots ,u_{m-1},\mathbf{S'}\rangle$, with
$\mathbf{S'}$ being the interpretation of term $S$ in $\M'$ and $\langle
u_1,\ldots ,u_{m-1},\mathbf{S}\rangle\in\mathbf{F_i}$.

Then, $\M'$ is a model of $\bt$. By definition, the interpretation
of situation--independent predicates and functions is the same in
$\M'$ and $\M$.
\end{lem}

This lemma can be easily proved by induction over possible syntactic form of sentences in $\bt$.
If $S$ and $S'$ are two situation terms and $\bt$ is a set of
formulas uniform in $S$, then we denote by $\bt(S'/S)$ the set of
formulas obtained from $\bt$ by replacing every occurrence of $S$
with $S'$. This notation will be extensively used in Section
\ref{Section_Progression}. Obviously, $\bt(S'/S)$ is uniform in
$S'$.

\medskip

The initial theory $\DD_{S_0}$ of $\DD$ is defined as an arbitrary set
of sentences in the signature $\si$ that are uniform in the situation
constant $S_0$. Throughout the paper, we assume that $\DD_{S_0}$
is a theory in (any fragment of) second-order logic that can
be translated into a set of sentences of first-order logic uniform in
$S_0$. In particular, $\DD_{S_0}$ can include both an ABox and a TBox in an appropriate
Description Logic, as argued in \cite{GuSoutchanski2010,YehiaDL2012}.

Next, for every n-ary action function $A\in\si$, a theory $\DD_{ap}$
includes an axiom of the form
$$\forall \ \bar{x}, s \ \big(\, Poss(A(\bar{x}),s)\leftrightarrow
\Pi_A(\bar{x},s)\, \big)\, ,$$ where $\Pi_A(\bar{x},s)$ is a formula uniform
in $s$ with free variables among $\bar{x}$ and $s$.
Informally, $\Pi_A(\bar{x},s)$ characterizes preconditions for
executing the action $A$ in the situation $s$. No other formulas are in
$\DD_{ap}$.

\medskip

\textbf{Example \ref{BWonly} (continuation)}.
The following is the precondition axiom for $move(x,y,z)$:
\medskip

$
\hspace{-0.5cm}
\begin{array}{l}
Poss(move(x,y,z),s) \liff \  Block(x)\land Block(y)\land 
Block(z)\land On(x,y,s)\land \\ \hspace{6.5cm} Clear(x,s)\land Clear(z,s)\land x \not= z
\end{array}
$

\medskip

The action $move(x,y,z)$ is possible iff $x,y,z$ are blocks, $x$ is located 
on $y$ in situation $s$, and both the block $x$ that is to be moved, and
a destination block $z$ are not occupied by any other blocks. Notice 
the preconditions do not allow moving a block back to the same location where
it was before.

\medskip

\noindent Finally, for every fluent $F\in\si$, a theory $D_{ss}$
contains an axiom of the form
$$\forall\ \bar{x},a,s\ \big(\, F(\bar{x},do(a,s))\leftrightarrow \gamma_F^+(\bar{x},a,s)\vee
\ F(\bar{x},s)\wedge\neg \gamma_F^-(\bar{x},a,s)\, \big) \hspace{0.6cm} (\dagger),$$

\noindent 
specifying a condition $\gamma_F^+(\bar{x},a,s)$ when fluent $F$ becomes true 
in situation $do(a,s)$, or when $F$ remains true in situation $do(a,s)$ 
if it is true in $s$, unless another condition $\gamma_F^-(\bar{x},a,s)]$ holds.
Here, $\gamma_F^+$ is a disjunction of formulas of the
form $[\exists \bar{y}]
(a=A^+(\bar{t})\wedge\phi^+(\bar{x},\bar{y},s))$, where $A^+$ is an action
function, $\bar{t}$ is a (possibly empty) vector of object terms with
variables at most among $\bar{x}$ and $\bar{y}$, and $\phi^+$ is a
formula uniform in $s$ with variables at most among $\bar{x}$,
$\bar{y}$, and $s$. We write $[\exists \bar{y}]$ to show that
$\exists \bar{y}$ is optional; it is present only if $\bar{t}$
includes $\bar{y}$ or if $\phi$ has an occurrence of $\bar{y}$.
The formula $\phi^+$ is called a  {\it positive context
condition} meaning that $A^+(\bar{t})$ makes the fluent $F$ true if
this context condition holds in $s$, but otherwise, $A^+(\bar{t})$
has no effect on $F$. Similarly, $\gamma_F^-$ is a disjunction of
formulas of the form $[\exists \bar{z}]
(a=A^-(\bar{t'})\wedge\phi^-(\bar{x},\bar{z},s))$, where $A^-$ is an
action function, $\bar{t'}$ is a (possibly empty) vector of object
terms with variables at most among $\bar{x}$ and $\bar{z}$, and
$\phi^-$ is a formula uniform in $s$ with variables
 at most among $\bar{x}$, $\bar{z}$, and $s$.
 The formula $\phi^-$ is called a {\it negative context condition} meaning that
$A^-(\bar{t})$ makes the fluent $F$ false if this context condition holds in $s$,
but otherwise, $A^-(\bar{t})$ has no effect on $F$.
In the definition above, we assume that the empty disjunction is
equal to \textit{false}. No other formulas are in $\DD_{ss}$.
This completes the definition of $\DD_{ss}$.
Subsequently, the following will be useful.

\begin{de}[SSA and active position of an action]
The axioms of $\DD_{ss}$ in the form above are called \emph{successor state
axioms} (SSAs) of a basic action theory $\DD$.

An action function $f$ is said to be in \emph{active position} of some SSA
$\varphi\in\DD_{ss}$ if $f$ occurs either as $A^+$, or $A^-$ in the
definition of $\DD_{ss}$ above. 

We say that $\varphi\in\DD_{ss}$ is SSA \emph{for the fluent} $F$ if $F$ is
the fluent from the left-hand side of $\varphi$.
\end{de}

\textbf{Example \ref{BWonly} (continuation)}.
The following is the SSA for the fluent $EH(x,s)$ meaning
the height of a block $x$ is even, i.e., the number of blocks under $x$ is odd:
\medskip

$
\begin{array}{ll}
EH(x, do(a, s)) \leftrightarrow 
	& \exists y, z\big(a \!=\! move(x,y,z) \land \neg EH(z, s)\big)\ \lor\\
	& EH(x, s) \land \neg\exists y, z\big(a\! =\! move(x, y, z) \land EH(z, s)\big).
\end{array}
$

\medskip

Then formula $\neg EH(z,s)$ is a positive context condition. If it holds in $s$,
i.e., if the height of block $z$ is not even in a situation $s$, then in the situation
that results from moving $x$ from $y$ to a block $z$, the height of $x$ becomes 
even. But if the positive context condition does not hold in $s$, then 
$move(x,y,z)$ does not make the height of block $x$ even. Also, if the height of 
$x$ is even in $s$, then it remains even unless a block $x$ is moved from $y$ 
on top of $z$ and the height of $z$ is even in $s$. The formula $EH(z,s)$ is a 
negative context condition, i.e., if the height of block $z$ is even in $s$, 
then the action  $move(x,y,z)$ has a negative conditional effect on 
the fluent $EH(x,s)$ in the sense that this fluent becomes false in 
the situation that results from doing $move(x,y,z)$ in $s$. In this SSA, 
an action function $move$ occurs both as $A^+$ and $A^-$ on the right hand 
side of this SSA.

\smallskip

Following the consistency requirement on SSAs by Reiter
(see Proposition 3.2.6 in \cite{Reiter2001}), we require that if
an action function $f$ occurs in active position in some SSA for a fluent $F$, then $f$ is not in active position in either $\gamma^+_F$, or $\gamma^-_F$.
\commentout { 
The SSAs are obtained under an assumption that for each fluent $F$ the formula\\
$\begin{array}{c}
\ttab\ttab\ttab\ttab\ttab
\neg\exists \bar{x},a,s \big(\, \gamma_F^+(\bar{x},a,s)\land \gamma_F^-(\bar{x},a,s)\, \big)
\end{array}$\\
is entailed by background axiomatization of effects of actions on
fluents. } 
Informally, this means that an action cannot have both positive
and negative effects on $F$. \commentout{under the same context
condition}

Each SSA for a fluent $F$ completely defines the truth value of
$F$ in the situation $do(a,s)$ in terms of what holds in situation
$s$. Also, SSA compactly represents non-effects by quantifying
$\forall a$ over variables of sort action. Only action terms that
occur explicitly on the right-hand side of  SSA for a fluent
$F$ have effects on this fluent, while all other actions have no
effect.

We note that the original version of Reiter's situation calculus
admits functional fluents, e.g. functions having a vector of
arguments of sort object and one last argument of sort situation.
Reiter defines the notion of SSA for functional fluents \cite{Reiter2001}. 
Without loss of generality, we omit functional fluents in this paper. 

The following fundamental result, which will be used in our Theorem \ref{Teo_PreservationOfInseparability}, says that the initial theory together with the UNA is the core of any basic action theory, while the rest of the constituent theories may be considered as add-ons. 

\begin{prop}[Theorem 1 in \cite{PR99}]\label{Prop_RelativeSatisfiability} A basic
action theory $\Si\cup D_{una}\cup D_{S_0}\cup D_{ap}\cup D_{ss}$
is satisfiable iff $D_{una}\cup D_{S_0}$ is satisfiable.
\end{prop}

Suppose $\A_1,\cdots,\A_n$ is a sequence of ground action
terms, and $\varphi(s)$ is a formula with one free variable $s$ of
sort situation which is uniform in $s$. One of the most important
reasoning tasks in the situation calculus is the \textit{projection
problem}: that is, to determine whether

\centerline{$\DD\models \varphi(do(\A_n, do(\A_{n-1},
do(\cdots,do(\A_1,S_0))))).$}

\noindent Informally, $\varphi$ represents some property of
interest and  entailment holds iff this property is true in the
situation resulting from performing the sequence of actions
$\A_1,\cdots,\A_n$ starting from $S_0$.

Another basic reasoning task is the \textit{executability
problem}. Let

\centerline{$executable(do(\A_n, do(\A_{n-1},
do(\cdots,do(\A_1,S_0)))))$}

\noindent be an abbreviation of the formula \\
\centerline{
   $Poss(\A_1,S_0)\nd \bigwedge_{i=2}^{n}Poss(\A_i, do(\A_1, do(\cdots,do(\A_{i-1},S_0))).$
} Then, the executability problem is to determine whether
\\
\centerline{ $\DD\models executable(do(\A_n, do(\A_{n-1},
do(\cdots,do(\A_1,S_0)))))$, }

\noindent i.e. whether it is possible to perform the sequence of
actions starting from $S_0$.

Planning and high-level program execution are two important
settings, where the executability and projection problems arise
naturally. \textit{Regression} is a central computational
mechanism that forms the basis for an automated solution to the
executability and projection tasks in the situation calculus
(\cite{Reiter2001}). Regression requires reasoning backwards: a
given formula

\centerline{$\varphi(do(\A_n, do(\A_{n-1},
do(\cdots,do(\A_1,S_0)))))$}

\noindent
is recursively transformed into a logically equivalent formula by
using SSAs until the resulting formula has only occurrences of
the situation term $S_0$. It is easy to see that regression
becomes computationally intractable if the sequence of actions
grows indefinitely \cite{GuSoutchanski2010}. In this case, an alternative
to regression is progression, which provides forward-style
reasoning. The initial theory $\DD_{S_0}$ is updated to take into
account the effects of an executed action. Computing the progression of a
given theory $\DD_{S_0}$ requires forgetting facts in $\DD_{S_0}$
which are no longer true after executing an action. The closely
related notions of progression and forgetting are discussed in the
next sections.

\begin{de}[local-effect SSA and
$\cal{BAT}$]\label{Def_LocallEffect} An SSA $\varphi\in\DD_{ss}$
for the fluent $F$ is called \emph{local-effect} if the set of arguments of
every action function in active position of $\varphi$ contains all
object variables from $F$. A basic action theory is said to be
local-effect if every axiom of $\DD_{ss}$ is a local-effect SSA.
\end{de}


Local-effect $\cal{BAT}$s are a well-known\footnote{The phrase \textit{local-effect} 
actions first appeared in \cite{LiuLevesqueIJCAI2005}, but it was motivated by actions 
with \textit{simple effects} defined in the paper \cite{LinKR04}, where simple
effects are understood similar to the Def.~\ref{Def_LocallEffect}.}
class of theories, for which the operation of progression (Section \ref{Section_Progression})
can be computed effectively  \cite{ProgressionLocalEffect}, without regard to
decidability of the underlying theory $\DD_{S_0}$. They are
special in the sense that the truth value of each fluent defined
by a local-effect SSA can change only for objects explicitly named as arguments
of the executed action. Therefore, in local-effect $\cal{BAT}$s, each action can change
only finitely many ground fluent atoms. This allows for computing forgetting (the
operation considered in Section \ref{Section_Forgetting}) efficiently. 
Informally speaking, forgetting erases from $\DD_{S_0}$ those finitely many 
fluent atoms which changed after executing an action. 

\medskip

\textbf{Example \ref{BWonly} (continuation)}.  Observe that in 
the Blocks World example considered above, the action $move$ has only 
local effects on the fluents $On$ and $Clear$. As an informal example of
an action that has global effects, consider the action $drive(t,l_1,l_2)$ of 
driving a truck $t$ loaded with boxes from one location $l_1$ to another 
location $l_2$. Consider also the fluent  $At(x,l,s)$ that holds if an object 
$x$  is at a location $l$ in $s$. Observe that this action would have a global
effect on location of all boxes loaded on the truck since these boxes are not
named explicitly in the action function $drive(t,l_1,l_2)$, but the SSA for
$At(x,l,s)$ would have a $\forall$-quantifier over the object argument $x$. 
Therefore, the truth value of $At(x,l,s)$ changes not only for $t$, but also for
other objects not mentioned in $drive(t,l_1,l_2)$. It would be awkward  to include all 
the boxes loaded in $t$ as arguments of this action. For this reason, axioms for 
the logistics domain should include actions with global effects on the fluents.

\medskip

Before we proceed to a discussion of component properties under forgetting 
(Section \ref{Section_Forgetting}) and to progression of initial theories 
(Section \ref{Section_Progression}), 
we consider an example that helps to illustrate the notion of $\cal{BAT}$ and the advantages of decomposition of its initial theory. Our example combines
the simplified Blocks World (BW) with a kind of Stacks World.
A complete axiomatization of BW modelled as a finite collection
of finite chains can be found in \cite{CookLiuJLC2003}. In this example, and
subsequently, we resort to the common situation calculus convention that free 
variables in $\cal{BAT}$ axioms are implicitly taken to be universally quantified at front.

\begin{example}[A running example of $\cal{BAT}$]\label{BWexample}
{\rm 	
The blocks-and-stacks-world consists of a finite set of blocks and a finite set of other entities. 
Blocks can be located on top of each other, while other entities can be either in a heap of unlimited capacity, or can be organized in stacks. There is an unnamed manipulator that can
move a block from one block to another, provided that there is nothing on the top of the blocks. It can also put an entity from the heap upon a stack with a named top element, or move the top element of a stack into the heap. For stacking/unstacking operations we adopt the push/pop terminology and use the unary predicate $Block$ to distinguish between blocks and other entities. We use the following action functions and relational fluents to axiomatize this example as a local-effect $\cal{BAT}$ in SC.

\smallskip

\noindent{\bf Actions}\nopagebreak
\samepage
\vspace*{-1.5mm}
\begin{itemize}
\setlength{\itemsep}{0.0\itemsep} 
\item
     $move(x,y,z)$: Move block $x$ from block $y$ onto block $z$, provided
     both $x$ and $z$ are clear.
\item
	\!$\!push(x,\!y)$: Stack entity $x$ from the heap on top of entity~$y$.
\item
	$pop(x)$: Unstack entity $x$ into the heap, 
	provided $x$ is the top element and is not in the heap.
\end{itemize}

\noindent{\bf Fluents}
\vspace*{-1.5mm}
\begin{itemize}
\setlength{\itemsep}{0.0\itemsep} 
\item
  $On(x,z,s)$: Block $x$ is on block $z$, in situation $s$.
\item
  $Clear(x,s)$: Block $x$ has no other blocks on top of it in  $s$.
\item
  $Top(x,s)$: Entity $x$ is the top element of a stack in $s$.
\item 
  $Inheap(x,s)$:  Entity $x$ is in the heap in situation $s$.
\item 
  $Under(x,y,s)$: Entity $y$ is directly under $x$ in a stack in situation $s$.
\end{itemize}

The sub-theories of the basic action theory are defined as follows.

\medskip

\noindent{\bf Successor state axioms (theory $\DD_{ss}$)}
\nopagebreak
\samepage

$
\begin{array}{l}
\hspace{-0.5cm}
On(x,z,do(a,s)) \liff \ \exists y (a\!=\! move(x,y,z)) \lor On(x,z,s)\land 
     \neg \exists y ( a\! =\! move(x,z,y))\\[1ex]
\hspace{-0.5cm}
Clear(x,do(a,s)) \liff\ \exists y, z ( a\!=\!move(y,x,z) 
\land\\ 
\hspace{3.8cm} On(y,x,s) ) \lor\, Clear(x,s)\!\land
 \neg \exists y,\!z (a\! =\! move(y,z,x))\\[1ex]
\hspace{-0.5cm}
Inheap(x,do(a,s)) \liff \ a\! =\! pop(x) \lor
       Inheap(x,s) \land \neg \exists y (a\! =\! push(x,y) )\\[1ex]
\hspace{-0.5cm}
Top(x,do(a,s)) \liff\  \exists y 
(\, a\!=\!push(x,y)\, ) \lor
\exists y (\, a\!=\!pop(y)\land Under(y,x,s)\, ) \ \lor \\
\hspace{4.9cm}
Top(x,s) \land  a\! \not=\! pop(x)\land \neg \exists y (a\!=\!push(y,x))\\[1ex]
\hspace{-0.5cm}
Under(x,y,do(a,s)) \liff \ a\!=\! push(x,y)\ \lor
	Under(x,y,s)\land a\! \not=\! pop(x)
\end{array}
$

\medskip

The first axiom is saying that a block $x$ will be on top of a block $z$
after moving $x$ from another block $y$ onto $z$, or if $x$ was already on $z$
and it was not moved elsewhere. The second axiom is saying that $x$ will become
clear, i.e., there will be no blocks on top of $x$ after moving the block $y$
that was previously on top of $x$ onto another block $z$. Otherwise, if a block $x$
was already clear, it remains clear unless some block $y$ will be moved from 
the block $z$ onto the block $x$. The third axiom asserts that an entity $x$ is
in a heap once it has been removed from a stack, or if it was already in a heap, 
and it was not stacked on top of another entity $y$. In the fourth axiom, when
an entity $x$ is stacked upon an entity $y$, $x$ become the new top. 
Also, it becomes the top, when $x$ was located under some $y$ that was removed
into a heap. Otherwise, an entity $x$ remains on the top unless it was unstacked
or buried under by stacking another entity $y$ onto $x$. In the last fifth axiom,
an entity $y$ will be under another entity $x$ after stacking $x$ on top of $y$,
or $y$ remains under $x$ after any action that does not remove $x$ into a heap.
It is easy to observe that all these SSAs are local-effect, and we will exploit this fact later in our paper.

\medskip

\noindent{\bf Action precondition axioms (theory $\DD_{ap}$)}

$
\hspace{-0.5cm}
\begin{array}{l}
Poss(move(x,y,z),s) \liff \  Block(x)\land Block(y)\land 
Block(z)\land On(x,y,s)\land \\ \hspace{6.5cm} Clear(x,s)\land Clear(z,s)\land x \not= z
\end{array}\\[0.5ex]
\hspace{-0.5cm}
\begin{array}{l}
Poss(push(x,y),s) \liff \neg Block(x)\land \neg Block(y)\land
Top(y,s)\land Inheap(x,s)
\end{array}\\[0.5ex]
\hspace{-0.5cm}
\begin{array}{l}
Poss(pop(x),s) \liff \neg Block(x)\land Top(x,s) \\ 
\end{array}
$\\

The precondition axioms are self-explanatory. The action $move(x,y,z)$ is possible
in any situation $s$ where a block $x$ is located on top of a block $y$, both
$x$ and a destination block $z$ are clear  (i.e., not obstructed by any blocks
on top of them) and $x$ is different from $z$. The last condition precludes 
moving $x$ on top of itself. According to the second precondition axiom, it is
possible to stack $x$ on $y$ in any situation $s$, if $x$ and $y$ are entities
which are not blocks, $y$ is the top of a stack, and $x$ is in a heap.
The opposite operation of unstacking $x$ is possible if and only if $x$ is a
top entity in situation $s$.

\medskip

\noindent{\bf Initial Theory ($\DD_{S_0}$)} 
is defined as the set of axioms\footnote{Some of these axioms, e.g., 
the second axiom, remain true after executing any of the possible actions, but 
this fact is irrelevant to the purposes of this example.}  using object constants $\{A,B,C\}$: 

\smallskip

$\begin{array}{l}
\neg\exists y  On(y,x,S_0)\land \exists y  On(x,y,S_0)\land \neg Inheap(x,S_0)\!\rightarrow\! Clear(x,S_0)\\
\exists y \ On(x,y,S_0)\rightarrow Block(x) \\ 
(Top(x,S_0)\lor Inheap(x,S_0))\rightarrow\neg Block(x)\\
On(A,B,S_0)\!\land Block(B)\!\land Block(C)\!\land Clear(A,S_0)\!\land Clear(C,S_0)
\end{array}$

\medskip

\noindent{\bf Unique names axioms for actions and objects (theory
$\DD_{una}$)} is the set of unique names axioms for all pairs of object constants
and action functions used above.

\medskip

\noindent Then $\Si\cup\DD_{una}\cup\DD_{ap}\cup\DD_{ss}\cup\DD_{S_0}$ is
the resulting local-effect basic action theory.

\medskip

Notice that all fluents are syntactically related in $\DD_{S_0}$, so purely
syntactic techniques fail to decompose $\DD_{S_0}$ into components sharing no fluents.
However, $\DD_{ss}$ is the union of two theories with the intersection of 
signatures equal to $\{do\}$. The set of precondition axioms is also union of
two theories -- the first axiom by itself is one of them, and the conjunction of
the second and third axioms is another one -- with the intersection of
signatures equal to $\{Poss,Block\}$.  At the same time, the initial theory 
$\DD_{S_0}$ is $\D$--decomposable for $\D=\{Block,S_0\}$ into 
two distinct $\D$--inseparable components:

\medskip

\noindent
$
\begin{array}{l}	
\neg\exists y \  On(y,x,S_0)\land \exists y \  On(x,y,S_0)\rightarrow Clear(x,S_0)\\
\exists y \ On(x,y,S_0)\rightarrow Block(x) \\ 
On(A,B,S_0)\!\land Block(B)\!\land Block(C)\!\land Clear(A,S_0)\!\land Clear(C,S_0)
\end{array}
$

\smallskip

and

\smallskip

\noindent$\begin{array}{l}	
(Top(x,S_0)\lor Inheap(x,S_0))\rightarrow\neg Block(x)\\
\exists x \ Block(x)
\end{array}$
} 
\end{example}

This example is continued after Theorem \ref{Teo_PreserveCompLocalEffect} in Section \ref{Section_Progression}, where we will show that the progression for $\cal{BAT}$s of this kind preserves both decomposability and inseparability of the decomposition components.


\section{Properties of Forgetting}\label{Section_Forgetting} There are two basic types of forgetting considered in the literature: forgetting a signature and forgetting a ground atom. As will be explained in Section \ref{Section_Progression}, progression of $\mathcal{BAT}$s is closely related to forgetting. In particular, computing progression of a local-effect $\mathcal{BAT}$ involves forgetting a set of ground atoms representing facts that are no longer true after an action execution. Thus, in order to understand the behavior of the component properties of theories under progression, one needs to first examine their relationship to forgetting, which is the purpose of this section. Although we are focused on forgetting ground atoms, the counterpart results for signature forgetting often come for free and are therefore included into this section. Moreover, they help to see the difference between the two types of forgetting, which contributes to a better understanding of this operation wrt the component properties, which we believe would be of interest to a broader audience in the literature. To emphasize broader applicability of these results, we consider a general first- and second-order logic setting in the remainder of this section. \\

Let us define a relation
on structures as follows. Let $\si$ be a signature or a ground
atom and $\M$, $\M'$ be two many--sorted structures. Then we set
$\M\sim_{\si}\M'$ if:
\begin{itemize}
\item $\M$ and $\M'$ have the same domain for each sort;
\item $\M$ and $\M'$ interpret all symbols which do not occur in $\si$
identically;
\item if $\si$ is a ground atom $P(\bar{t})$ then $\M$ and $\M'$
agree on interpretation $\bar{u}$ of $\bar{t}$ and for every
vector of elements $\bar{v}\neq \bar{u}$,
we have $\M\models P(\bar{v})$ iff $\M'\models
P(\bar{v})$. 
\end{itemize}
Obviously, $\sim_{\si}$ is an equivalence relation.

\medskip

The following notion summarizes the well-known Definitions 1 and 7 in \cite{ForgetIt}.

\begin{de}[Forgetting an atom or signature]\label{De_Forgetting}
Let $\bt$ be a theory in $\LL$ and $\si$ be either a signature, or
some ground atom. A set $\bt'$ of formulas in a fragment of
second-order logic is called the result of \emph{forgetting} $\si$ in
$\bt$ (denoted by $\forget(\bt, \si)$) if for any structure $\M'$,
we have $\M'\models\bt'$ iff there is a model $\M\models\bt$ such
that $\M\sim_\si\M'$.
\end{de}

It is known that $\forget(\bt,$ $\si)$ always exists, i.e. it is second-order
definable, for a finite set of formulas $\bt$ in $\LL$ and a finite
signature or a ground atom $\si$ (see \cite{ForgetIt}, or Section 2.1
in \cite{ProgressionLocalEffect}). On the other hand, the
definition yields $\bt\models\forget(\bt, \si)$; thus,
$\forget(\bt, \si)$ is a set of second-order consequences of $\bt$ which
suggests that it may not always be definable in the logic, where
$\bt$ is formulated, and it may not be finitely axiomatizable in
this logic, even if $\bt$ is so.

\begin{fact}[Basic properties of
forgetting]\label{Fact_BasicPropertiesOfForgetting} If $\si$ and
$\pi$ are signatures or ground atoms and $\bt$, $\bt'$ are
theories in $\LL$ then:
\begin{itemize}
\item $\forget(\bt, \si\cup\pi)\equiv\forget(\forget(\bt,
\si),\pi)$ \ (if $\si$ and $\pi$ are signatures) 
\item $\forget(\forget(\bt, \si),\pi)\equiv\forget(\forget(\bt,
\pi),\si)$ 
\item $\forget(\forget(\bt, \si), \si)\equiv
\forget(\bt, \si)$ 
\item $\forget(\bt, \si)\equiv\bt$ \ if $\si$ is
a signature with $\si\cap\sig(\bt)=\varnothing$, or a ground atom
with predicate not contained in $\sig(\bt)\ $ 
\item $\forget(\bt\cup\bt',\si)\not\equiv\forget(\bt,\si)\cup\forget(\bt',\si)$
(see Example \ref{Ex_ComponentwiseForgetting}) 
\item $\forget(\varphi\vee\psi, \si)\equiv \forget(\varphi,
\si)\vee\forget(\psi,\si)$ \ (if $\varphi,\psi$ are formulas in $\!\LL$).
\end{itemize}
\end{fact}
These properties either follow immediately from the definition, or from 
the results proven in \cite{ForgetIt}.

\begin{prop}[Signature of forget($\bt$,$\si$)]\label{Prop_SignatureOfForgetting}
Let $\bt$ be a theory in $\LL$, $\si$ a signature (or a ground
atom) and let $\forget(\bt,\si)$ be a set of
formulas in a language $\LL'$, a fragment of second-order logic
with PIP. Then $\forget(\bt,\si)$ is logically equivalent in
$\LL'$ to a set of formulas in the signature $\sig(\bt)\setminus\si$
($\sig(\bt)$, respectively).
\end{prop}

\begin{proof} 
We consider the case when $\si$ is a signature; the case of
a ground atom being proved analogously. Assume that
$\si\cap\sig(\forget(\bt,\si))\neq\varnothing$. Denote by
$\forget(\bt,\si)^*$ a ``copy" of the set of formulas
$\forget(\bt,\si)$, where each symbol from
$\si\cup\allowbreak [\sig\allowbreak(\forget\allowbreak(\bt,\allowbreak \si))\setminus\allowbreak \sig(\bt)]$ 
is uniquely replaced with a fresh symbol, not present in
$\sig(\forget\allowbreak(\bt,\si))$. We claim that
$\forget(\bt,\si)^*\models_{\LL'}\forget(\bt,\si)$. There is
nothing to prove if $\forget(\bt,\allowbreak \si)^*$ is unsatisfiable. Note
that, by definition of forgetting, $\forget(\bt,\si)^*$ and
$\forget(\bt,\allowbreak \si)$ are satisfiable iff $\bt$ is. Let us assume
that $\bt$ is satisfiable. Take an arbitrary model
$\M^*\models\forget(\bt,\si)^*$; then there exists a model
$\M'\models\forget(\bt,\allowbreak\si)$ which agrees on
$\sig(\forget(\bt,\si)^*)$ with $\M^*$ and interprets symbols from
$\si\cup[\sig(\forget(\bt,$ $\si))\setminus\sig(\bt)]$ equally to
the interpretation of the corresponding fresh symbols in $\M^*$.
Therefore, we may assume that $\M^*\sim_{\si}\M'$. By definition
of forgetting, there is a model $\M\models\bt$ such that
$\M'\sim_{\si}\M$, hence $\M^*\sim_{\si}\M$ and
$\M^*\models\forget(\bt,\si)$. We have
$\forget(\bt,\si)^*\models_{\LL'} \forget(\bt,\si)$ and
$\sig(\forget(\bt,\si)^*)\cap\allowbreak  \sig($ $\forget(\bt,\allowbreak \si))$
$\subseteq\sig(\bt)\setminus\si$. By PIP, there is a set of
formulas $\Theta$ in signature $\sig(\bt)\setminus\si$ such that
$\forget(\bt,\si)^*\models_{\LL'}\Theta$ and $\Theta\models_{\LL'}
\forget(\bt,\si)$. Note that $\forget(\bt,\si)^*$
$\models_{\LL'}\Theta$ yields
$\forget(\bt,\si)\models_{\LL'}\Theta$, because every model of
$\forget(\bt,\si)$ can be expanded to a model of
$\forget(\bt,\si)^*$ and the reduct of this model onto (a subset
of) $\sig(\bt)\setminus\si$ suffices to satisfy $\Theta$. Thus, we
conclude that $\forget(\bt,\si)$ is equivalent to $\Theta$. 
\end{proof}

\begin{cor}
Let $\bt$ be a theory in $\LL$ having PIP and $\si$ a
signature. Then $\bt\equiv\forget(\bt,\si)$ iff $\bt$ is
equivalent to a set of formulas in the signature
$\sig(\bt)\setminus\si$.
\end{cor}

We note that the similar statement does not hold when $\si$ is a
ground atom. It follows from Proposition
\ref{Prop_SignatureOfForgetting} that in case $\si$ is a
signature, $\forget(\bt, \si)$ axiomatizes the class of reducts of
models of $\bt$ onto the signature $\sig(\bt)\setminus\si$. Clearly,
if $\bt$ is a theory in language $\LL$, then $\forget(\bt, \si)$
may not be in $\LL$, however it is always expressible in
second-order logic if $\bt$ is finitely axiomatizable (we note that
second-order logic has PIP). For the
case when $\si$ is a signature, $\forget(\bt, \si)$ is known as
$\sig(\bt)\setminus\si$--uniform interpolant of $\bt$ wrt the language $\LL$ and second-order queries, that is, wrt the pair
($\LL$, second-order logic), see Definition 13 in
\cite{FormalPropMod} and Lemma 39 in \cite{ConsExtEL} for a
justification. In other words, $\bt$ and $\forget(\bt, \si)$
semantically entail the same second-order formulas in signature
$\bt\setminus\si$.

If $\si$ is a ground atom $P(\bar{t})$ then, by definition, for
any model $\M\models\bt$, $\forget(\bt,\allowbreak \si)$ must have two
``copies'' of $\M$: a model with the value of $P(\bar{t})$ false
and a model where this value is true. Let $\LL$ be first-order
logic. In contrast to forgetting a signature, for any recursively
axiomatizable theory $\bt$ in $\LL$ and a ground atom $\si$, one
can effectively construct the set of formulas $\forget(\bt, \si)$
in $\LL$ such that $\forget(\bt, \si)$ is finitely axiomatizable
iff $\bt$ is. This follows from Theorem 4 in \cite{ForgetIt},
where it is shown that forgetting a ground atom $P(\bar{t})$ in a
theory $\bt$ can be computed by simple syntactic manipulations:
\begin{itemize}
\item for an axiom $\varphi\in\bt$, denote by
$\varphi[P(\bar{t})]$ the result of replacing every occurrence of
atom $P(\bar{t'})$ (with $\bar{t'}$ a term) by formula
$[\bar{t}=\bar{t'}\wedge P(\bar{t})]\vee[\bar{t}\neq\bar{t'}\wedge
P(\bar{t'})]$

\item denote by $\varphi^+[P(\bar{t})]$ the formula
$\varphi[P(\bar{t})]$ with every occurrence of the ground atom
$P(\bar{t})$ replaced with \textit{true} and similarly, denote by $\varphi^-[P(\bar{t})]$ the formula $\varphi[P(\bar{t})]$
with $P(\bar{t})$ replaced with \textit{false}

\item then $\forget(\bt, P(\bar{t}))$ is equivalent to
$(\bigwedge_{\varphi\in\bt}\varphi^+[P(\bar{t})]) \ \bigvee \
(\bigwedge_{\varphi\in\bt}\varphi^-[P(\bar{t})]).$
\end{itemize}

The disjunction corresponds to the union of two classes of models
obtained from models of $\bt$, with the ground atom $P(\bar{t})$
interpreted as \textit{true} and \textit{false}, respectively.
This fact is important for effective computation of progression
for local-effect $\cal{BAT}$s mentioned in Section
\ref{Section_Progression}. 

\begin{example}[Forgetting a ground atom]
Consider a theory $\bt=\{\varphi\}$, where $\varphi=\neg P(c)$, i.e., $P(c)$ is false in every model of $\bt$. Consider forgetting $P(c)$ in $\bt$. By the (semantic) definition of forgetting, the set of models of $\forget(\bt, P(c))$ consists of models of $\bt$ and those models, in which $P(c)$ is true. Therefore, any structure is a model of $\forget(\bt, P(c))$. Now consider the syntactic definition of forgetting given above. We have 
\begin{equation*}
\varphi[P(c)]=\neg([c=c \wedge P(c)]\vee [c\neq c\wedge P(c)])
\end{equation*}
thus, $\varphi^+[P(c)]\equiv false$, $\varphi^-[P(c)]\equiv true$ and hence, $\forget(\bt, P(c))\equiv \varphi^+[P(c)] \vee \varphi^-[P(c)]$ is a tautology. 

Now consider forgetting $P(c)$ in the theory $\bt=\{\varphi\}$, where $\varphi=\forall x P(x)$. By the (semantic) definition of forgetting, any structure, in which $P$ is true on every element, except possibly, the interpretation of $c$, is a model of $\forget(\bt, P(c))$. By the syntactic definition of forgetting we have:
\begin{equation*}
\varphi[P(c)]= \forall x \ (\ [x=c\wedge P(c)]\vee [x\neq c\wedge P(x)] \ )
\end{equation*}
thus,  $\varphi^+[P(c)]=\forall x \ (\ x=c \vee [x\neq c \wedge P(x)] \ )$, $\varphi^-[P(c)]=\forall x \ (\  x\neq c \wedge P(x) \ )$. Since $\varphi^-[P(c)]\equiv false$, we obtain $\forget(\bt, P(c))\equiv \varphi^+[P(c)]$ and hence, $\forget(\bt, P(c))\equiv \forall x\ (\ x\neq c\ \rightarrow P(x) \ )$.
\end{example}

We note that in case a theory $\bt$ is
finitely axiomatizable, computing $\forget(\bt, P(\bar{t}))$ in
the way above doubles the size of theory in the worst case, due to the disjunction. It is
sometimes necessary to consider forgetting of some set $S$ of
ground atoms in a theory $\bt$. This is equivalent to iterative
computation of forgetting of atoms from $S$ starting from the theory
$\bt$ (the order on atoms can be chosen arbitrary as noted in Fact
\ref{Fact_BasicPropertiesOfForgetting}). However, it is important
to note that the size of the resulting theory is $O(2^{|S|}\times
|\bt|)$, where $|S|$ is the number of atoms in $S$ and $|\bt|$ is
the size of $\bt$.

\begin{prop}[Interplay of forgetting and entailment]\label{Prop_ForgettingDerivDiagram} Let $\bt$ and
$\bt_1$ be two sets of formulas in $\LL$, with $\bt\models\bt_1$,
and $\si$ be a signature or a ground atom. Then the following
holds:
\begin{center}
\makebox[12cm]{

\put(-50,0){{\large $\bt$}} 

\put(0,0){{\large $\models$}} 

\put(0,-40){{\large $\models$}} 

\put(50,0){{\large $\bt_1$}} 

\put(-50,-20){\rotatebox[origin=c]{270}{{\large $\models$}}} 

\put(50,-20){\rotatebox[origin=c]{270}{{\large $\models$}}}

\put(-90,-40){{\large $\forget(\bt,\si)$}} 

\put(30,-40){{\large $\forget(\bt_1,\si)$}} 

}

\end{center}
\end{prop}

\begin{proof}
Follow the diagram starting from the top-left column.
 By definition of forgetting, every model of $\bt$ is a
model of $\forget(\bt,\si)$, so we have $\bt\models\forget(\bt,\si)$ as shown 
in the left column of the diagram.  Similarly, $\bt_1\models\forget(\bt_1,\si)$
in the right column of the diagram. To prove entailment at the bottom  we rely
on entailments in the columns and on the given entailment at the top, i.e.,
we navigate the diagram up from the bottom-left corner, then move right, 
and finally go down to the bottom-right expression. Let $\M'$ be an arbitrary
model of $\forget(\bt,\si)$. Then there is a model $\M\models\bt$
such that $\M\sim_\si\M'$. Since $\bt\models\bt_1$, we have
$\M\models\bt_1$, so we conclude that
$\M'\models\forget(\bt_1,\si)$, because $\M$ is a model satisfying
the conditions of Definition \ref{De_Forgetting} for $\bt_1$ and $\M'$. 
Thus, we proved entailment shown in the bottom row of the diagram. 
\end{proof}

\begin{prop}[Preservation of consequences under forgetting]\label{Prop_ConsecPreservationUnderForgetting}
Let $\bt$ be a theory in $\LL$ and $\si$ be either a signature or
a ground atom. Let $\varphi$ be a formula such that either
$\sig(\varphi)\cap\si=\varnothing$ (in case $\si$ is a signature),
or which does not contain the predicate from $\si$ (if $\si$ is a
ground atom). Then $\bt\models\varphi$ iff $\forget(\bt,
\si)\models\varphi$.
\end{prop}

\begin{proof} 
From Proposition \ref{Prop_ForgettingDerivDiagram}, we have
$\bt\models \forget(\bt, \si)$, thus $\forget(\bt,
\si)\models\varphi$ yields $\bt\models\varphi$. Now let
$\bt\models\varphi$ and assume there is a model $\M'$ of
$\forget(\bt, \si)$ such that $\M'\not\models\varphi$. By
definition of forgetting, there exists a model $\M$ of $\bt$ such
that $\M\sim_\si \M'$, i.e. $\M$ and $\M'$ have the same universe
and may differ only on interpretation of signature $\si$ (ground
atom $\si$). By the condition on signature of $\varphi$, then $\M$
is not a model of $\varphi$, which contradicts
$\bt\models\varphi$. 
\end{proof}

\medskip

Now we answer the question when inseparability is preserved under forgetting. This is important for our research, since we are interested in preservation of inseparability under progression, the operation which relies on forgetting in local-effect $\mathcal{BAT}$s. We demonstrate that it is important to distinguish between forgetting something in $\Delta$ (the common symbols of theories) or outside of the shared signature. While Proposition \ref{Prop_ConsecPreservationUnderForgetting} shows that the situation is simple in the latter case, it is apriory unclear, whether the same holds in the former. Example \ref{Ex_InseparabilityLostForgettingAtom} demonstrates that this is not true, while the accompanying Propositions \ref{Prop_InseparabilityPreservationUnderForgetting}, \ref{Prop_ModelInseparabilityPreservationUnderForgetting} describe the cases when this situation can be recovered. We believe that giving the accompanying positive results is important in order to provide a big picture to the reader. Proposition \ref{Prop_InseparabilityPreservationUnderForgetting} shows that signature forgetting (the arguably more frequently used type of forgetting in the literature) preserves inseparability, while Proposition \ref{Prop_ModelInseparabilityPreservationUnderForgetting} tackles this question from another perspective. It shows that semantic inseparability, the property also well studied in the literature, is the stronger form of inseparability, which is invariant under forgetting.

Observe that by Proposition
\ref{Prop_ConsecPreservationUnderForgetting} and the first item of Fact \ref{Fact_BasicPropertiesOfForgetting}, when studying
preservation of $\D$--inseparabi\-lity of two sets of formulas for a
signature $\D$, it is sufficient to consider the case of
forgetting a subset of $\D$ or a ground atom with the predicate from
$\D$, respectively.

\begin{prop}[Preservation of $\D$--insep. under signature forgetting]\label{Prop_InseparabilityPreservationUnderForgetting}
$\ \  $ Let $\LL$ have PIP and $\bt_1$ and $\bt_2$ be two
$\D$--inseparable sets of formulas in $\LL$ with
$\sig(\bt_1)\cap\sig(\bt_2)=\D$, for a signature $\D$. Let $\si$ be
a subsignature of $\D$ and $\forget(\bt_1,\si)$ and
$\forget(\bt_2,\si)$ be sets of formulas of $\LL$. Then
$\forget(\bt_1,\si)$ and $\forget(\bt_2,\si)$ are
$\D$--inseparable.
\end{prop}

\begin{proof} 
Let $\varphi$ be a formula with $\sig(\varphi)\subseteq\D$
such that $\forget(\bt_1,\si)\models\varphi$. By Proposition
\ref{Prop_SignatureOfForgetting}, we may assume that for $i=1,2$
the signature of $\forget(\bt_i,\si)$ is a subset of
$\sig(\bt_i)\setminus\si$. We depict the direction of the proof in the figure below.
\begin{center}

\makebox[12cm]{

\put(-85,0){{\large $\bt_1\models\bt'_1$}} 

\put(10,0){{\large $\Rightarrow$}} 



\put(65,0){{\large $\bt_2\models\bt'_1$}} 

\put(-81,-20){\rotatebox[origin=c]{270}{{\large \rotatebox[origin=c]{180}{$\Rightarrow$}}}} 

\put(68,-20){\rotatebox[origin=c]{270}{{\large \rotatebox[origin=c]{360}{$\Rightarrow$}}}}

\put(-140,-40){{\large $\forget(\bt_1,\si)\models\bt'_1\models\varphi$}} 


\put(40,-40){{\large $\forget(\bt_2,\si)\models\bt'_1\models\varphi$}} 

}

\end{center}

We start with the lower-left corner of the diagram and navigate up, 
then right, and finally down.
As $\forget(\bt_1,\si)\models\varphi$, by PIP, there is a set of
formulas $\bt_1'$ with
$\sig(\bt_1')\subseteq\allowbreak  \sig(\forget(\bt_1,\allowbreak \si))\cap\allowbreak \sig(\varphi)$
$\subseteq\D\setminus\si$ such that
$\forget(\bt_1,\si)\models\bt_1'$ and $\bt_1'\models\varphi$.
Then, by Proposition
\ref{Prop_ConsecPreservationUnderForgetting}, we have
$\bt_1\models\bt_1'$. This proves entailment in the top-left corner.
Since $\bt_1$ and $\bt_2$ are
$\D$--inseparable and $\sig(\bt_1')\subseteq\D$, we obtain
$\bt_2\models\bt_1'$. Therefore, the top-right entailment holds. Again, since
$\sig(\bt_1')\cap\si=\varnothing$, by Proposition
\ref{Prop_ConsecPreservationUnderForgetting}, we conclude that
$\forget(\bt_2,\si)\models\bt_1'$ and thus,
$\forget(\bt_2,\si)\models\varphi$. 
\end{proof}

\medskip

The following example demonstrates that a similar result does not
hold under forgetting a ground atom with the predicate from $\D$.

\begin{example}[$\D$--inseparability lost under forgetting a ground atom]\label{Ex_InseparabilityLostForgettingAtom} We give an
example of a logic $\LL$, sets of formulas $\bt_1$, $\bt_2$ in
$\LL$, and a signature $\D=\sig(\bt_1)\cap\sig(\bt_2)$ such that
$\bt_1$ and $\bt_2$ are $\D$--inseparable, but
$\forget(\bt_1,R(c,c))$ and $\forget\allowbreak(\bt_2,\allowbreak R(c,c))$ are not, for a
ground atom $R(c,c)$ with a predicate $R\in\D$.
Let $\LL$ be Description Logic $\cal{ELO}^\bot$, i.e. the
sub-boolean logic $\cal{EL}$ augmented with nominals and the
bottom concept $\bot$. Let $\Si=\{R,a,c\}$ be signature, where $R$
is a role name (binary predicate) and $a,c$ are nominals (i.e.
constants). Define a set of formulas $\bt_1$ in the signature $\Si$
as $\{\{a\}\sqcap \{c\}\sqsubseteq\bot, \ \{c\}\sqsubseteq \exists
R.\{a\}, \ \top\sqsubseteq \exists R.\top \}$.
Set $\D=\{R,c\}$ and consider the set of formulas
$\bt_2=\{\top\sqsubseteq\exists R.\top, \ Taut(c)\}$, where
$Taut(c)$ is a tautology with the nominal $c$ (e.g., the
formula $\{c\}\sqsubseteq\top$). We have
$\sig(\bt_1)\cap\sig(\bt_2)=\D$ and it is easy to check that
$\bt_2$ is equivalent to $\Cons(\bt_1,\D)$ in the logic
$\cal{ELO}^\bot$; thus, $\bt_1$ and $\bt_2$ are $\D$--inseparable.
Now consider $\forget(\bt_1, R(c,c))$ and $\forget(\bt_2,$
$R(c,c))$ as sets of formulas in second-order logic (we assume the
standard translation of formulas of $\cal{ELO}^\bot$ into the
language of second-order logic). We verify that they are not
$\D$--inseparable and the formula $\top\sqsubseteq\exists R.\top$
is the witness for this. By definition of $\bt_1$, we have
$\forget(\bt_1, R(c,c))\models\bt_1$, since any model of $\bt_1$
with a changed truth value of the predicate $R$ on the pair $\langle
c,c \rangle$ is still a model of $\bt_1$. On the other hand,
$\forget(\bt_2, R(c,c))\not\models \top\sqsubseteq\exists R.\top$,
because $\bt_2$ has the one--element model $\M$, where $R$ is
reflexive (on the sole element corresponding to $c$). Hence, by
definition of forgetting, the one-element model $\M'$ with $R$
false on the pair $\langle c,c \rangle$ must be a model of
$\forget(\bt_2, R(c,c))$, but obviously, $\M'\not\models
\top\sqsubseteq\exists R.\top$.
\end{example}

It turns out that the preservation of inseparability under forgetting
a ground atom requires rather strong model-theoretic conditions
like (*) in Proposition
\ref{Prop_ModelInseparabilityPreservationUnderForgetting} below.
Specialists might notice that (*) is equivalent to
\textit{semantic $\D$--inseparability} of the initial sets of
formulas (see Definition 11 in \cite{FormalPropMod}) which is very hard to decide from the computational point of view
(see Theorem 3 in \cite{ConsExtALC}, Lemma 40 in
\cite{ConsExtEL}). Nevertheless, there are practically useful restrictions under which the complexity becomes feasible \cite{ModelThInseparability-PositiveResults}. Semantic $\D$--inseparability is 
stronger than the notion of inseparability given in Definition
\ref{Def_Inseparable}: it means that the theories are indistinguishable by second-order formulas. On the other hand, Proposition
\ref{Prop_ModelInseparabilityPreservationUnderForgetting} says
that whenever there is a chance to satisfy (*) for two given sets
of formulas, one does not need to check it again after forgetting
something in their common signature. To compare condition (*) with
Example \ref{Ex_InseparabilityLostForgettingAtom}, note that the
aforementioned one-element model of $\bt_2$ does not expand to a model
of $\bt_1\cup\bt_2$.

\begin{prop}[Preservation of $\D$--inseparability under forgetting]\label{Prop_ModelInseparabilityPreservationUnderForgetting}
Let $\bt_1$ and $\bt_2$ be two sets of formulas in $\LL$, with
$\sig(\bt_1)\cap\sig(\bt_2)=\D$, for a signature $\D$, which satisfy
the following condition (*): for $i=1,2$, any model of $\bt_i$ can
be expanded to a model of $\bt_1\cup\bt_2$.
Then:
\begin{itemize}
\item $\bt_1$ and $\bt_2$ are $\D$--inseparable;
\item for $\si$ a signature or a ground atom, $\forget(\bt_1,\si)$
and $\forget(\bt_2,\si)$ satisfy (*) as well.
\end{itemize}
\end{prop}

\begin{proof} 
$\D$--inseparability is the immediate consequence of (*):
if $\varphi$ is a formula with $\sig(\varphi)\subseteq\D$,
$\bt_1\models\varphi$, but $\bt_2\not\models\varphi$, then there
is a model $\M_2$ of $\bt_2$ such that $\M_2\not\models\varphi$.
Then there is an expansion $\M$ of $\M_2$ such that
$\M\models\bt_1\cup\bt_2$, $\M|_{\sig(\bt_1)}\models\bt_1$, but
$\M|_{\sig(\bt_1)}\not\models\varphi$, a contradiction. Now let
us verify that for $i=1,2$, any model of $\forget(\bt_i,\si)$ can
be expanded to a model of
$\forget(\bt_1,\si)\cup\forget(\bt_2,\si)$. For instance, let
$\M_2'$ be a model of $\forget(\bt_2,\si)$. Consider a model
$\M_2$ of $\bt_2$, such that $\M_2\sim_\si\M_2'$, and expand it to a
model $\M$ of $\bt_1\cup\bt_2$. Then by definition of forgetting,
there must be a model $\M'\models\forget(\bt_1,\si)$ with
$\M'\sim_\si\M$, which agrees with $\M_2'$ on $\si$ (if $\si$ is a
signature), or on the predicate of $\si$ (if $\si$ is a ground
atom). By construction, $\M'$ is an expansion of $\M_2'$ and thus
a model for $\forget(\bt_1,\si)\cup\forget(\bt_2,\si)$. 
\end{proof}

\medskip

Let $\bt_1$ and $\bt_2$ be two sets of formulas in $\LL$, with
$\sig(\bt_1)\cap\sig(\bt_2)=\D$, for a signature $\D$, and let $\si$
be either a subsignature of $\D$ or a ground atom with the predicate
from $\D$. It is known that in general, forgetting $\si$  may not
be distributive over union of sets of formulas. The entailment
$\forget(\bt_1\cup\bt_2,\si)\models\forget(\bt_1,\si)\cup\forget(\bt_2,\si)$
holds by Proposition \ref{Prop_ForgettingDerivDiagram}, but
Example \ref{Ex_ComponentwiseForgetting} below easily shows that
even strong semantic conditions related to modularity do not
guarantee the reverse entailment. On the other hand, forgetting
something outside of the common signature of $\bt_1$ and $\bt_2$
is distributive over union, as formulated in Corollary
\ref{Cor_ForgettingWithSpecComponent} which is a consequence of the
criterion in Proposition
\ref{Prop_CriterionComponentwiseForgetting} and is used in the proof of one of our main results, Theorem \ref{Teo_PreserveCompLocalEffect}.

\begin{example}[Failure of componentwise forgetting in $\D$]\label{Ex_ComponentwiseForgetting}
Let \ $\LL$ \ be first-order logic and $\D=\{P,c\}$ be the signature
consisting of a unary predicate $P$ and a constant $c$. Define
theories $\bt_1$ and $\bt_2$ as: $\bt_1=\{A\rightarrow P(c)\}$,
$\bt_2=\{P(c)\rightarrow B\}$, where $A,B$ are nullary predicate
symbols. We have $\sig(\bt_1)\cap\sig(\bt_2)=\D$ and for $i=1,2$,
any model of $\bt_i$ can be expanded to a model of
$\bt_1\cup\bt_2$. Clearly, $\bt_1$ and $\bt_2$ are
$\D$--inseparable and for $i=1,2$, $\Cons(\bt_i,\D)$ is the set of
tautologies in $\D$.
By definition of forgetting, for $i=1,2$, $\forget(\bt_i,P(c))$ is
a set of tautologies and thus,
$\forget(\bt_1,P(c))\cup\forget(\bt_2,P(c))\not\models\forget(\bt_1\cup\bt_2,P(c))$,
because $\forget(\bt_1\cup\bt_2,P(c))\models A\rightarrow B$ (by
Proposition \ref{Prop_ConsecPreservationUnderForgetting}).
For the case of forgetting a signature, say a nullary predicate $P$,
it suffices to consider $\D=\{P\}$ and theories
$\bt_1=\{A\rightarrow P\}$, $\bt_2=\{P\rightarrow B\}$, where
$A,B$ are nullary predicates.
\end{example}

\begin{prop}[A criterion for componentwise forgetting]\label{Prop_CriterionComponentwiseForgetting}
Let $\bt_1$ and $\bt_2$ be two sets of formulas and $\si$
either a signature or a ground atom. Then the following statements are
equivalent:
\begin{itemize}
\item
$\forget(\bt_1,\si)\cup\forget(\bt_2,\si)\models\forget(\bt_1\cup\bt_2,\si)$
\item for any two models $\M_1\models\bt_1$ and
$\M_2\models\bt_2$, with $\M_1\sim_\si\M_2$, there exists a model
$\M\models\bt_1\cup\bt_2$ such that $\M\sim_\si\M_i$, for some
$i=1,2$.
\end{itemize}
\end{prop}

\begin{proof}
Note in the second condition, the requirement
$\M\sim_\si\M_i$ for some $i=1,2$ is equivalent to
$\M\sim_\si\M_i$ for all $i=1,2$, by transitivity of $\sim_\si$.
($\Rightarrow$): Let $\M_1\models\bt_1$ and $\M_2\models\bt_2$ be
models with $\M_1\sim_\si\M_2$. Then there are models $\M_1'$ and
$\M_2'$ such that for $i=1,2$, $\M_i'\models\forget(\bt_i,\si)$
and $\M_i'\sim_\si\M_i$. Then, by transitivity of $\sim_\si$, for
all $i,j=1,2$ we have $\M_i'\sim_\si\M_j$ and thus,
$\M_i'\models\forget(\bt_j,\si)$. Then
$\M_1'\models\forget(\bt_1\cup\bt_2,\si)$, so there exists a model
$\M\models\bt_1\cup\bt_2$ such that $\M\sim_\si\M_1'$ and hence,
$\M\sim_\si\M_1$. ($\Leftarrow$): Let $\M'$ be a model of
$\forget(\bt_1,\si)\cup\forget(\bt_2,\si)$. There exist models
$\M_1$ and $\M_2$ such that for $i=1,2$, $\M_i\models\bt_i$ and
$\M_i\sim_\si\M'$. Then $\M_1\sim_\si\M_2$, hence, there must be a
model $\M$ of $\bt_1\cup\bt_2$ with $\M\sim_\si\M_i$ for some
$i=1,2$. Then we obtain that $\M\sim_\si\M'$ and thus, by
definition of forgetting, $\M'$ is a model of
$\forget(\bt_1\cup\bt_2,\si)$. 
\end{proof}

\medskip

To compare this criterion with Example
\ref{Ex_ComponentwiseForgetting}, observe that there exist models
$\M_1\models\bt_1$ and $\M_2\models\bt_2$ with common domain such
that $\M_1\models A\wedge P(c)\wedge\neg B$ and $\M_2\models
A\wedge\neg P(c)\wedge\neg B$. Thus, ${\M_1}\sim_{P(c)}\M_2$,
however, there does not exist a model $\M$ of $\bt_1\cup\bt_2$
such that ${\M}\sim_{P(c)}\M_i$, for some $i=1,2$. Neither $\M_1$,
nor $\M_2$ is a model for $\bt_1\cup\bt_2$.

\begin{cor}[Forgetting in the scope of one component]\label{Cor_ForgettingWithSpecComponent} Let $\bt_1$ and
$\bt_2$ be two sets of formulas, with
$\sig(\bt_1)\cap\sig(\bt_2)=\D$, for a signature $\D$, and $\si$ be
either a subsignature of $\sig(\bt_1)\setminus\D$ or a ground atom
with the predicate from $\sig(\bt_1)\setminus\D$. Then
$\forget(\bt_1\cup\bt_2,\si)$ is equivalent to
$\forget(\bt_1,\si)\cup\bt_2$. Moreover, if $\bt_1$
and $\bt_2$ are $\D$--inseparable, then so are
$\forget(\bt_1,$ $\si)$ and $\bt_2$.
\end{cor}

\begin{proof} 
Note that by the choice of $\si$, $\bt_2$ is
equivalent to $\forget(\bt_2,\si)$ and thus, by Proposition
\ref{Prop_ForgettingDerivDiagram}, it suffices to verify the
entailment
$\forget(\bt_1,\si)\cup\forget(\bt_2,\si)\models\forget(\bt_1\cup\bt_2,\si)$.
If there are models $\M_1\models\bt_1$ and $\M_2\models\bt_2$,
with $\M_1\sim_\si\M_2$, then in fact,
$\M_1\models\bt_1\cup\bt_2$, by the choice of $\si$ and definition
of $\sim_\si$. Thus, the criterion from Proposition
\ref{Prop_CriterionComponentwiseForgetting} obviously yields the
required entailment. It remains to note that
$\D$--inseparability of $\forget(\bt_1,\si)$ and
$\forget(\bt_2,\si)$ follows from the choice of $\si$, Proposition
\ref{Prop_ConsecPreservationUnderForgetting}, and
$\D$--inseparability of $\bt_1$ and $\bt_2$. 
\end{proof}

\medskip

In general, the results of this section prove that the operation
of forgetting does not behave well wrt the modularity
properties of the input, since they are logic-dependent. Stronger model-theoretic conditions on
the input are needed due to the model-theoretic nature of
forgetting.


\section{Properties of Progression}\label{Section_Progression}
We have considered some component properties of forgetting. It
turns out that the operation of progression is closely related to
forgetting in initial theories. However, in case of progression, we can not
restrict ourselves to working with initial theories only; we need also
to take into account information from successor state axioms. The
aim of this section is to study component properties of
progression wrt different forms of SSAs and common signatures
$\D$s (\textit{deltas}) of components of initial theories. We
will consider local-effect SSAs discussed in
\cite{ProgressionLocalEffect} and \textit{deltas}, which do not
contain fluents.

We use the following notations further in the paper. For a ground
action term $\alpha$ in the language of the situation calculus, we
denote by $S_\alpha$ the situation term $do(\alpha,S_0)$. 
To define progression, we introduce an equivalence relation on
many-sorted structures in the situation calculus signature. For two
structures $\M$, $\M'$ and a ground action $\alpha$, we set
$\M\sim_{S_\alpha}\M'$ if:
\begin{itemize}
\item $\M$ and $\M'$ have the same sorts for action and object;
\item $\M$ and $\M'$ interpret all situation-independent predicate
and function symbols identically;
\item $\M$ and $\M'$ agree on interpretation of all fluents at
$S_\alpha$, i.e., for every fluent $F$ and every variable
assignment $\theta$, we have $\M,\theta\models
F(\bar{x},S_\alpha)$ iff $\M',\theta\models F(\bar{x},S_\alpha)$.
\end{itemize}

That is, if $\M\sim_{S_\alpha}\M'$ then the structures $\M$ and $\M'$ are allowed to differ in sorts for situation and interpretations of fluents at situation terms, not equal to $S_\alpha$. 

Note that a similar notation with $\sim$ is used to denote the equivalence relation on models from Definition \ref{De_Forgetting} of forgetting. The two notations are easily distinguished depending on the context and are standard in the literature, therefore we adopt both of them in our paper.

\begin{de}[Progression, modified Definition 9.1.1 in \cite{Reiter2001}]\label{Def_Progression}
Let $\DD$ be a basic action theory with unique name axioms
$\DD_{una}$ and the initial theory $\DD_{S_0}$, and let $\alpha$
be a ground action term. A set $\DD_{S_\alpha}$ of formulas in a
fragment of second-order logic is called \emph{progression} of
$\DD_{S_0}$ wrt $\alpha$ if it is uniform in the situation term
$S_\alpha$ and for any structure $\M$, $\M$ is a model of
$\Sigma\cup \DD_{ss}\cup \DD_{ap} \cup \DD_{una}\cup \DD_{S_\alpha}$ iff there is a model $\M'$ of
$\DD$ such that $\M\sim_{S_\alpha}\M'$.
\end{de}


Below, we use $\DD_{S_\alpha}$ to denote progression of the initial
theory wrt the action term $\alpha$, if the context of $\cal{BAT}$ is
clear. We sometimes abuse terminology and call progression not
only the theory $\DD_{S_\alpha}$, but also the operation of
computing this theory (when the existence of an effective
operation is implicitly assumed). It can be seen (Theorem 2 in \cite{LinReiter1997} and Theorem 2.10 in \cite{ProgressionLocalEffect})  that progression always exists, i.e., is
second-order definable, if the signature of $\cal{BAT}$ is finite and
the initial theory $\DD_{S_0}$ is finitely axiomatizable.
On the other hand, by the definition, for any $\cal{BAT}$ $\DD$, we have $\DD\models \DD_{S_\alpha}$
and, similarly to the operation of forgetting, it is possible to
provide an example (see Definition 2, Conjecture 1, and Theorem 2
in \cite{FONonDefinabilityofProgression}), when the progression
$\DD_{S_\alpha}$ is not definable (even by an infinite set of
formulas) in the logic in which $\DD$ is formulated.

\smallskip

To understand the notion of progression intuitively, note the
following. The progression $\DD_{S_\alpha}$ is a set of consequences
of $\cal{BAT}$ that are uniform in the situation term $S_\alpha$;
it can be viewed as the {\it strongest postcondition} of the precondition $\DD_{S_0}$ wrt
the action $\alpha$. Thus, informally, $\DD_{S_\alpha}$ is all the information about the
situation $S_\alpha$ implied by $\cal{BAT}$. 
This is guaranteed by the model-theoretic property with the
relation $\sim_{S_\alpha}$ in the definition. Recall that the
initial theory of $\cal{BAT}$ describes information in the initial
situation $S_0$ and SSAs are essentially the rules for computing the new truth
values of fluents that change after performing actions. Thus, progression
$\DD_{S_\alpha}$ can be viewed as minimal ``modification'' of the initial
theory obtained after executing the action $\alpha$. In particular,
the initial theory of $\cal{BAT}$ can be replaced with
$\DD_{S_\alpha}(S_\alpha/S_0)$ (recall the notation from Section
\ref{Sect_SituationCalculus}) which gives a new $\cal{BAT}$, with
$S_\alpha$ as the new initial situation.  Let $\varphi(s)$ be a formula uniform in a 
situation variable $s$. To solve the projection problem for $\varphi(S_\alpha)$,
i.e., to find whether $\varphi(S_\alpha)$  holds in the situation $S_\alpha$ wrt 
$\cal{BAT}$ $\DD$, one might wish to compute progression $\DD_{S_\alpha}$ and then 
check whether $\DD_{una}\cup\DD_{S_\alpha}\models\varphi(S_\alpha)$ holds (or
equivalently, whether 
$\DD_{una}\cup\DD_{S_\alpha}(S_0/S_\alpha)\models\varphi(S_0)$ holds). 
By Proposition \ref{Prop_RelativeSatisfiability}, this is
equivalent to $\DD\models\varphi(S_\alpha)$, so this  progression-based
approach solves the projection problem for $\varphi(S_\alpha)$. This helps to demonstrate
why progression may be useful.

\smallskip

Consequently, of interest are cases when progression can be computed effectively
as a theory in the same logic that is used to formulate underlying $\DD_{S_0}$.
The well-known approach
is to consider the local-effect $\cal{BAT}$s (recall Definition
\ref{Def_LocallEffect}) in which progression can be obtained by
just a syntactic modification of the initial theory $\DD_{S_0}$
with respect to SSAs. This approach is based on effective
forgetting of a finite set of ground atoms (extracted from SSAs)
in the initial theory of $\cal{BAT}$. Recall the well--known
observation from Section \ref{Section_Forgetting} that, given a
theory $\bt$ (in an appropriate logic $\LL$), forgetting a finite set of
ground atoms in $\bt$ can be computed effectively by
straightforward syntactic manipulations with the axioms of $\bt$.
Thus, the cornerstone of computing progression in the local-effect
case is to extract effectively the set of ground atoms from SSAs
that need to be forgotten. Subsequently, in $\DD_{S_\alpha}$, they
are replaced with new values of fluents, which are
computed from SSAs. An interested reader may consult the whole
paper \cite{ProgressionLocalEffect}, while here we only introduce
necessary notations from Definition 3.4 of
\cite{ProgressionLocalEffect}, which will be used in Theorem
\ref{Teo_PreserveCompLocalEffect}.

\smallskip

Let $\DD$ be a $\cal{BAT}$ with a set $\DD_{ss}$ of SSAs, an
initial theory $\DD_{S_0}$, and a unique name theory
$\DD_{una}$, and let $\alpha$ be a ground action term. 
Take a generic SSA ($\dagger$) for the fluent $F$ (see Section \ref{Sect_SituationCalculus}) and replace
there an action variable $a$ with the action term $\alpha$. Then, use unique name
axioms for actions to replace equalities (or negations of equalities) between 
action functions with equalities (or negations of equalities, respectively) 
between object arguments. After that, apply the usual FO logic equivalences 
to eliminate existential quantifiers inside 
$\gamma^+_F(\bar{x},\alpha,s),\gamma^-_F(\bar{x},\alpha,s)$, if any. 
Recall these are formulas uniform in $s$ that appear on the right-hand side of 
a generic SSA ($\dagger$). Observe that in a local-effect SSA, when one 
substitutes a ground action term  $A(\bar{b_x},\bar{b_z})$ for a variable $a$ in
the formula $[\exists\bar{z}]. a\!=\!A(\bar{x},\bar{z})\land\phi(\bar{x},\bar{z},s)$,
applying UNA for actions yields
$[\exists\bar{z}]. \bar{x}\!=\!\bar{b_x}\land \bar{z}\!=\!\bar{b_z}\land
\phi(\bar{x},\bar{z},s)$, and applying
$\exists z(z\!=\!b\land\phi(z))\equiv \phi(b)$ repeatedly results in the logically
equivalent formula  $\bar{x}\!=\!\bar{b_x}\land \phi(\bar{x},\bar{b_z},s)$.
In a transformed SSA that is obtained after doing all these simplifications,
it is convenient to consider all object constants appearing in equalities 
between object variables and constants. These represent values where the fluent
$F$ changes. To compute the new value of the fluent it is sufficient to instantiate
object variables of $F$ with the corresponding constants. Denote

$\begin{array}{l} \D_F = \{\bar{t} \ \mid \ \bar{x}=\bar{t} \
\text{appears in} \ \gamma^+_F(\bar{x},\alpha,s) \ \text{or} \
\gamma^-_F(\bar{x},\alpha,s) \ \text{in a transformed SSA} \\ \hspace{1.1cm} \text{ for } \ F
 \ \text{instantiated with} \ \alpha \ \text{and equivalently rewritten wrt} \ \DD_{una} \},\\
\Omega(s) = \{F(\bar{t},s) \ \mid \ \bar{t}\in\D_F\}.
\end{array}$

\medskip

Consider $\Omega(s)$  and notice that $\Omega(S_0)$ is a finite set of ground atoms to be
forgotten. According to
Fact~\ref{Fact_BasicPropertiesOfForgetting}, forgetting several
ground atoms can be accomplished consecutively in any order.

An \textit{instantiation} of $\DD_{ss}$ wrt $\Omega(S_0)$, denoted
by $\DD_{ss}[\Omega(S_0)]$, is the set of formulas of the form:
$$F(\bar{t},do(\alpha,S_0))\leftrightarrow
\gamma_F^+(\bar{t},\alpha,S_0)\vee \ F(\bar{t},S_0)\wedge\neg
\gamma_F^-(\bar{t},\alpha,S_0).$$
These formulas represent instantiations of the transformed SSAs with object constants where
the fluents change. 
Observe that $\DD_{ss}[\Omega(S_0)]$ effectively defines new values for those fluents,
which are affected by the action $\alpha$. However, these definitions use fluents
wrt $S_0$, which may include fluents to be forgotten. For this reason, forgetting
should be performed not only in $\DD_{S_0}$, but in $\DD_{ss}[\Omega(S_0)]$ as well.

\begin{prop}[Theorem 3.6 in
\cite{ProgressionLocalEffect}]\label{Prop_ProgressionForLocalEffectCase}
In the notations above, the following is a progression of
$\DD_{S_0}$ wrt $\alpha$ in the sense of Definition
\ref{Def_Progression}:
$$\DD_{S_\alpha}=\ \forget\big(\DD_{ss}[\Omega(S_0)]\cup\DD_{S_0},\Omega(S_0)\big)\ (S_\alpha/S_0).$$
\end{prop}

This formula demonstrates that progression is a set of formulas obtained after
forgetting old values of fluents in the initial theory and in instantiation of 
transformed SSAs that provide new values of fluents, and then replacing  $S_0$ 
with $S_\alpha$. Thus, computing a progression in a local-effect $\cal{BAT}$ 
is an effective syntactic transformation of the initial theory, which
leads to the \textit{unique} form of the updated theory
$\DD_{S_\alpha}$. This fact will be used in Theorem
\ref{Teo_PreserveCompLocalEffect}. It is important to realize that
this transformation can lead to an exponential blow-up of the initial
theory, as noted after Theorem 3.6 in
\cite{ProgressionLocalEffect}, due to the possible exponential
blow-up after forgetting a set of ground atoms. This is not a
surprise, because even in propositional logic, forgetting a symbol
in a formula is essentially the elimination of a ``middle term"
(introduced by Boole), which results in the disjunction of two
instances of the input formula \cite{LinAIjournal2001}. As a
consequence, forgetting may result in a formula that is roughly
twice as long as the input formula. It is important to realize
that the exponential blowup is not inevitable in the case of
progression. As shown in \cite{ProgressionLocalEffect}, there are
practical classes of the initial theories for which there is no
blow-up and the size of the progressed theory is actually linear wrt
the size of the initial theory.

\medskip

\textbf{Example \ref{BWonly} (continuation)}.
As was discussed before, all SSAs in this example are local effect. 
Instantiate the action variable $a$ in the SSAs with a ground action 
$move(C_1,C_2,C_3)$. Then, we get: 

\medskip
$
\hspace{-0.4cm}
\begin{array}{ll}
\! Clear(x, do(move(C_1,C_2,C_3), s)) \leftrightarrow & 
	\exists y, z\big(move(C_1,C_2,C_3)\! =\! move(y, x, z)\big) \lor\\
& \!\!\! Clear(x, s) \land \neg\exists y, z(move(C_1,C_2,C_3)\! =\! move(y, z, x)),\\
On(x, y, do(move(C_1,C_2,C_3), s)) \leftrightarrow & 
	\exists z \big(move(C_1,C_2,C_3)\! =\! move(x, z, y)\big)\ \lor \\
&  On(x,y,s) \land \neg\exists z\big(move(C_1,C_2,C_3)\! =\! move(x, y, z)\big).
\end{array}
$

\medskip
\noindent Applying UNA for actions yields the following axioms:

\medskip
$
\begin{array}{ll}
Clear(x, do(move(C_1,C_2,C_3), s)) \leftrightarrow & 
	\exists y,z \big(y\!=\!C_1\land x\!=\!C_2 \land z\!=\!C_3\big)\ \lor\\
&	Clear(x, s) \land \neg\exists y,z(y\!=\!C_1\land z\!=\!C_2\land x\!=\!C_3),\\
On(x,y,do(move(C_1,C_2,C_3),s)) \leftrightarrow & \exists z (x\!=\!C_1\land z\!=\!C_2\land y\!=\!C_3) \lor\\ 
&	On(x,y,s) \land \neg\exists z(x\!=\!C_1\land y\!=\!C_2\land z\!=\!C_3).
\end{array}
$

\medskip
\noindent Doing the equivalent first order simplifications yields the \textit{transformed} SSAs:

\medskip
$
\hspace{-0.2cm}
\begin{array}{ll}
Clear(x, do(move(C_1,C_2,C_3), s)) \leftrightarrow & 
	(x\!=\!C_2) \lor\ Clear(x, s) \land \neg (x\!=\!C_3),\\
On(x, y, do(move(C_1,C_2,C_3), s)) \leftrightarrow &  (x\!=\!C_1\land y\!=\!C_3) \lor 
	On(x,y,s) \land \neg(x\!=\!C_1\land y\!=\!C_2).
\end{array}
$

\medskip
The \ \textit{argument set} $\Delta_F$ for the fluent $F$ 
wrt a ground action $\alpha$ is a set of constants appearing in the transformed 
SSA for $F$  instantiated with $\alpha$.  For example, the set $\Delta_{clear}$ 
for the fluent $Clear(x,s)$ wrt a ground action $move(C_1,C_2,C_3)$ is $\{C_2,C_3\}$. 
For the fluent $On(x,y,s)$ this argument set $\Delta_{on}$ is 
$\{\langle C_1,C_3\rangle, \langle C_1,C_2\rangle \}$.
The\ \textit{characteristic set} $\Omega$ of a ground action $move(C_1,C_2,C_3)$ 
is a set of all ground atoms subject to change by this action. Therefore\\
$ \hspace*{0.5in}
\Omega(s)=\{Clear(C_2,s), Clear(C_3,s), On(C_1,C_3,s), On(C_1,C_2,s)\}$.\\
Notice that if block $C_3$ is clear at $s$, it no longer remains clear after doing
$move(C_1,C_2,C_3)$ action, but block $C_2$ will become clear. However, this action
has no effect on the property of $C_1$ being clear, and for this reason, $C_1$
is not in $\Delta_{clear}$ and not in the characteristic set $\Omega$.

Using these atoms to instantiate the transformed SSA, i.e., by replacing object
arguments with constants from $\Delta_F$, we obtain the set 
${\cal D}_{ss}[\Omega]$ of formulas representing new values of fluents, e.g.,

\medskip 

$\ \ 
On(C_1,C_3,do(move(C_1,C_2,C_3),S_0))\leftrightarrow\\
	\ttab\ttab\ttab\ttab\ttab
C_1\!=\!C_1\land C_3\!=\!C_3\ \lor\  On(C_1,C_3,S_0)\land \neg(C2\!=\!C_3\land C_3\!=\!C_2).
$

\medskip

After the equivalent simplifications using UNA, the instantiated SSAs  wrt 
$\Omega(S_\alpha)$, where $S_\alpha=do(move(C_1,C_2,C_3),S_0)$ 
will be the following set:

\medskip

\ttab\ttab	 $\{Clear(C_2,S_\alpha), \neg Clear(C_3,S_\alpha), 
On(C_1,C_3,S_\alpha), \neg On(C_1,C_2,S_\alpha) \}$.

\medskip

Note that in this example ${\cal D}_{ss}[\Omega]$ are very simple,  but in a
general case, if a SSA includes context conditions, these axioms may include 
fluents wrt $S_0$. Finally, according
to Proposition \ref{Prop_ProgressionForLocalEffectCase}, to compute a progression
$\DD_{S_\alpha}$ of an initial theory $\DD_{S_0}$ for BW, we have to forget
all old values of the fluents from $\Omega(S_0)$ in the theory 
$\DD_{ss}[\Omega]\cup \DD_{S_0}$, and subsequently replace the situation
$S_\alpha$ with $S_0$.

\medskip

Now we are ready to formulate the results on component properties of
progression in terms of decomposability and inseparability. We
start with negative examples in which every $\cal{BAT}$ is
local-effect and the initial theories are formulated in first-order logic.  
As the progression $\DD_{S_\alpha}$ is a set
of formulas uniform in some situation term $S_\alpha$, which may
occur in every formula of $\DD_{S_\alpha}$ (thus potentially
spoiling decomposability), we consider the mentioned decomposability and inseparability
properties regarding the theory $\DD_{S_\alpha}(S_0/S_\alpha)$ instead of
$\DD_{S_\alpha}$. Otherwise, in every result we would have to
speak of $\D\cup\sig(S_\alpha)$--decomposability of progression\commentout{
instead of just $\D$--decomposability}, since the symbols from
$\sig(S_\alpha)\! =\! \sig(do(\alpha,S_0))\! =\! \{do,S_0\} \cup \sig(\alpha)$
may occur in all components.

Consider a $\cal{BAT}$ $\DD$ with $\D$--decomposable initial
theory $\DD_{S_0}$ for a signature $\D$. The general definition
of a successor state axiom gives enough freedom to design examples
showing (non-)preservation of the decomposability property of
$\DD_{S_0}$ or inseparability of its components. Note that an SSA may
contain symbols that are not even present in $\sig(\DD_{S_0})$, or
symbols from both components of $\DD_{S_0}$ (if decomposition
exists). Therefore,
it makes sense to restrict our study to those $\cal{BAT}$s, where
SSAs have one of the well-studied forms, e.g., to
local-effect theories. It turns out that this form is still
general enough to easily formulate negative results demonstrating that the
aforementioned properties are not preserved without stipulations.

First, we provide an example showing that the
decomposability property of the initial theory can be easily lost
under progression. 
Next, we show that $\D$--inseparability of components of the initial theory
$\DD_{S_0}$ can be easily lost when fluents are present in $\D$
(see Example \ref{Ex_LossOfInseparability}). The third observation is
that even if there are no fluents in $\D$, some components of
$\DD_{S_0}$ can split after progression into theories which are no
longer inseparable (see Example \ref{Ex_SplitOfComponent}). All
observations hold already for local-effect $\cal{BAT}$s and follow
from the fact that some new information from
SSAs can be added to the initial theory after progression, which spoils its component
properties. We only need to provide a combination of an initial
theory with a set of SSAs that are appropriate for this purpose. The aim of
Theorem \ref{Teo_PreservationOfInseparability} following these
negative examples is to prove that if $\D$ does not contain
fluents and the components of $\DD_{S_0}$ do not split after
progression, then $\D$--inseparability is preserved after
progression under a slight stipulation which is caused only by
generality of the theorem and the non-uniqueness of progression in
the general case. This stipulation is avoided in Theorem
\ref{Teo_PreserveCompLocalEffect}, where we consider the class of
local-effect $\cal{BAT}$s. Recall that all free variables in axioms of $\cal{BAT}$s 
are assumed to be universally quantified.

\medskip

\begin{example}[Decomposability lost under progression]\label{Ex_DecompLostUnderProgression}
Consider basic action theory $\DD$, with
$\{F,P,A,c\}\subseteq\sig(\DD)$, where $F$ is a 
fluent, $P$ a predicate, $A$ an action function, and $c$ an object
constant. Let the theory $\DD_{ss}$ consist of the single axiom

\begin{equation*}
F(x,do(a,s))\leftrightarrow (a=A(x))\wedge P(x) \ \vee \ F(x,s)
\end{equation*}

\noindent 
and let the initial theory $\DD_{S_0}$ consist of two
axioms $\neg F(c,S_0)$ and $\exists x P(x)$. Clearly, $\DD_{S_0}$ is a $\varnothing$--decomposable.

Consider action $\alpha=A(c)$ and let us compute progression of $\DD_{S_0}$ wrt $\alpha$. We apply Proposition \ref{Prop_ProgressionForLocalEffectCase}, since $\DD$ is
local-effect. The instantiation of the SSA from $\DD_{ss}$ with $\alpha$ has the form 
$$F(x,do(A(c),s))\leftrightarrow (A(c)=A(x))\wedge P(x) \ \vee \ F(x,s)$$

for which equivalent rewriting wrt $\DD_{una}$ gives
$$F(x,do(A(c),s))\leftrightarrow (x=c)\wedge P(x) \ \vee \ F(x,s)$$

 Hence, we have $\Omega(S_0)=\{F(c,S_0)\}$ and 
$\DD_{ss}[\Omega(S_0)]=F(c,S_\alpha)\leftrightarrow P(c) \ \vee \ F(c,S_0).$

Since $\neg F(c,S_0)\in\DD_{S_0}$, the theory $\DD_{ss}[\Omega(S_0)]\cup\DD_{S_0}$ is equivalent to $\{F(c,S_\alpha)\leftrightarrow P(c)\}\cup\DD_{S_0}$. By Proposition \ref{Prop_ProgressionForLocalEffectCase}, forgetting the ground atom $F(c,S_0)$ in this theory and substituting $S_0$ with $S_\alpha$ gives the theory $\DD_{S_\alpha}$, the progression of $\DD_{S_0}$ wrt $\alpha$. By using the definition of forgetting, it is easy to confirm that $\DD_{S_\alpha}$ is equivalent to $\{F(c,S_\alpha)\leftrightarrow P(c), \ \exists x P(x)\}$. One can verify that $\DD_{S_\alpha}$ (and also $\DD_{S_\alpha}(S_0/S_\alpha)$) is not $\D$--decomposable theory, for any $\D$. Notice that decomposability is lost, because fluent $F$ and predicate $P$ from different components of $\DD_{S_0}$ become related to each other after progression.
\end{example}

For a signature $\D$, with $S_0\in\D$, and an action $A(c)$, we
now give an example of a local-effect basic action theory $\DD$
with $\DD_{S_0}$, an initial theory $\D$--decomposable into finite
$\D$--inseparable components. This example shows that progression
$\DD_{S_\alpha}(S_0/S_\alpha)$ of $\DD_{S_0}$ wrt $A(c)$ (with
term $S_\alpha$ substituted with $S_0$) is finitely axiomatizable
and $\D$--decomposable, but the decomposition components are no
longer $\D$--inseparable, unless we allow them to be infinite.

\begin{example}[$\D$--inseparability is lost when fluents are in $\D$]\label{Ex_LossOfInseparability}
Consider a basic action theory $\DD$ with
$\{F,P,R,A,b,c\}\subseteq\sig(\DD)$, where $F$ is a
fluent, $P,R$ are predicates, $A$ an action function, and $b,c$ object constants. Let
$\D=\{F,R,S_0,c\}$ and define subtheories of $\DD$ as follows:\\

$\begin{array}{l}
\DD_{ss}=\{F(x,do(a,s))\leftrightarrow (a=A(x))\wedge P(x) \ \vee \ F(x,s)\} \ \ \text{(i.e. as in the previous example)}\vspace{0.2cm}\\ 
\DD_{S_0}=\DD_1\cup\DD_2, \text{with} \vspace{0.2cm}\\

\hspace{1.05cm} \DD_1=\{Taut(F,R,S_0,b), \ \neg F(c,S_0) \}, \ \text{where} \ Taut(F,R,S_0,b,c) \  \text{is a}\\ 
\hspace{1.05cm} \text{tautological formula in the signature} \ \{F,R,S_0,b,c\}, \text{which is uniform in}\ S_0\vspace{0.1cm}\\

\hspace{1.05cm} \DD_2=\{P(x)\rightarrow\exists y (R(x,y)\wedge P(y)), \
\neg F(c,S_0)\}.
\end{array}$\\

By the syntactic form, $\DD_{S_0}$ is $\D$--decomposable: we have
$\DD_{S_0}=\DD_1\cup\DD_2$, $\sig(\DD_1)\allowbreak\cap\sig(\DD_2)=\D$,
$\sig(\DD_1)\setminus\D=\{b\}$, and
$\sig(\DD_2)\setminus\D=\{P\}$. It is also easy to confirm that $\DD_1$ and $\DD_2$ are $\D$--inseparable. \smallskip


By Proposition \ref{Prop_ProgressionForLocalEffectCase} it is easy
to verify that the union of $\{Taut(F,R,S_0,b,c)\}$ and
$\DD_2'=(\DD_2\setminus\{\neg F(c,S_0)\})\cup\{\varphi\}$, where $\varphi=F(c,S_\alpha)\leftrightarrow P(c)$ is a progression ($\DD_{S_\alpha}$) of $\DD_{S_0}$ wrt $\alpha=A(c)$. \smallskip

By the syntactic form, $\DD_{S_\alpha}(S_0/S_\alpha)$ is a
$\D$--decomposable theory. On the other hand, we have
$\varphi\models F(c,S_\alpha)\rightarrow P(c)$, thus
\[	
\DD_2'(S_0/S_\alpha)\models\{ F(c,S_0)\rightarrow \exists y R(c,y)\, , \ \ 
F(c,S_0) \rightarrow \exists y\exists z[ R(c,y)\wedge R(y,z)]\, ,\ \ldots\}
\] 
This is an infinite set of formulas in signature $\D$. It follows from Fact \ref{Fact_ConsNotFinitelyAx} that this theory is not finitely
axiomatizable by formulas of first-order logic in signature $\D$
and it is easy to verify that $\DD_{S_\alpha}(S_0/S_\alpha)$ can not have a decomposition into finite $\D$--inseparable components.
\end{example}

Note that in the example above, the initial theory $\DD_{S_0}$
is in fact $\varnothing$--decompo\-sable with one signature
component equal to $\{b\}$ and the other component containing the
rest of the symbols. It is easy to see that the progression of
$\DD_{S_0}$ wrt $A(c)$ is $\varnothing$--decomposable as well. We
use tautologies in the example just to illustrate the idea that
information from SSA can propagate to the initial theory after
progression, thus making the components lose the inseparability
property. There is a plenty of freedom to formulate similar
examples with the help of non-tautological formulas which
syntactically ``bind'' symbols $F,R,S_0,b$ in the theory $\DD_1$. We
appeal to a similar observation in Example
\ref{Ex_SplitOfComponent}. 

\begin{example}[Split of a component and loss of $\D$--inseparability]\label{Ex_SplitOfComponent}
Consider $\cal{BAT}$ $\DD$,\newline with
$\{F_1,F_2,D,B,P,R,A,c\}\subseteq\sig(\DD)$, where $F_1, F_2$ are
fluents, $D,B,P,R$ predicates, $A$ an action function, and $c$ an object
constant. Let $\D=\{D,R,S_0\}$ and define the subtheories of $\DD$
as follows:\\

$\begin{array}{l}
\DD_{ss}=\{F_1(x,do(a,s))\leftrightarrow F_1(x,s)\wedge\neg(a=A(x)),
\ \ F_2(x,do(a,s))\leftrightarrow F_2(x,s)\}\vspace{0.2cm}\\

\DD_{S_0}=\DD_1\cup\DD_2, \text{where}\ \DD_1 \ \text{is the set of
formulas with occurrences of} \ D,R,S_0\text{:}\vspace{0.1cm}\\

\hspace{1.05cm} D(x)\vee R(x,y)\rightarrow F_1(c,S_0)
\vspace{0.1cm}\\

\hspace{1.05cm} D(x)\rightarrow P(x)
\vspace{0.1cm}\\

\hspace{1.05cm} P(x)\rightarrow\exists y (R(x,y)\wedge P(y))
\end{array}$\vspace{0.1cm}\\

and $\DD_2$ consists of the following three formulas (which also mention $D,R,S_0$):\\

$\begin{array}{l}
\hspace{1.05cm} D(x)\rightarrow B(x)
\vspace{0.1cm}\\

\hspace{1.05cm} B(x)\rightarrow\exists y (R(x,y)\wedge B(y))
\vspace{0.1cm}\\

\hspace{1.05cm} Taut(F_2,S_0), \text{a tautology in the signature}\ \{F_2,S_0\}, \text{uniform in} \ S_0. \ \text{Here,}\\ \hspace{1.05cm}  F_2 \ \text{is an auxiliary fluent introduced to have an occurrence of} \ S_0 \ \text{in} \ \DD_2.
\end{array}$\vspace{0.1cm}\\

By definition, $\DD_{S_0}$ is $\D$--decomposable into
$\D$--inseparable components $\DD_1$ and $\DD_2$. Note that
$\DD_{ss}\models \neg F_1(c,do(A(c),S_0))$, which is the result of
substitution of the ground action $A(c)$, situation constant $S_0$,
and object constant $c$ in SSA. \smallskip

Consider progression of $\DD_{S_0}$ wrt the action $\alpha=A(c)$. By
Proposition \ref{Prop_ProgressionForLocalEffectCase}, it is
equivalent to the theory
$\DD_{S_{\alpha}}=\DD_1'\cup\DD_1''\cup\DD_2'$, where $\DD_1'$ is
the set of the following formulas:\\

$\begin{array}{l}

\neg F_1(c,do(A(c),S_0))\vspace{0.2cm}\\

Taut(D,R), \ \text{a tautological formula in the
signature } \ \{D,R\} \ \text{which is uniform }

\text{in} \ S_\alpha,

\end{array}$\\

\noindent $\DD_1''$ is the set of formulas:\\

$\begin{array}{l}
D(x)\rightarrow P(x)\vspace{0.1cm}\\

P(x)\rightarrow\exists y (R(x,y)\wedge P(y))\vspace{0.1cm}\\

Taut(F_2,S_\alpha), \ \text{a tautological formula in the
signature } \ \{F_2,do,A,c,S_0\} \ \text{which}\\

\text{is uniform in} \ S_\alpha

\end{array}$\\

\noindent and $\DD_2'$ is the theory $\DD_2$ with every occurrence of $S_0$
substituted with $S_\alpha$. \smallskip

Clearly, $\DD_{S_{\alpha}}(S_0/S_\alpha)$ is $\D$--decomposable. Note that after progression the component $\DD_1$ is ``split'' into $\DD_1'(S_0/S_\alpha)$ and $\DD_1''(S_0/S_\alpha)$
and these theories are not $\D$--inseparable (similarly,
$\DD_1'(S_0/S_\alpha)$ and $\DD_2'(S_0/S_\alpha)$). By
Fact \ref{Fact_ConsNotFinitelyAx}, it can be shown that they can not be made
$\D$--inseparable while remaining finitely axiomatizable.
\end{example}

To formulate the theorems below, we let $\DD$ denote a $\cal{BAT}$
with the initial theory $\DD_{S_0}$, the set of successor state axioms
$\DD_{ss}$, and the unique name axioms $\DD_{una}$. Example \ref{Ex_LossOfInseparability} has resulted in the following definition.

\begin{de}[Fluent--free signature]
A signature $\D$ is called \emph{fluent--free} if no fluent (from the
alphabet of situation calculus) is contained in $\D$.
\end{de}


Theorem 4.2 complements Examples 4.3 and 4.4, which identify properties of decomposed actions theories causing loss of inseparability of components after progression. The theorem shows that if these properties are absent then inseparability is preserved. As we have
already seen in Example \ref{Ex_LossOfInseparability}, the initial
theory and progression may differ in consequences involving
symbols of fluents. Thus in general, preservation of
$\D$--inseparability can be guaranteed only for fluent-free
signatures $\D$, which is reflected in the conditions of the theorem.
Besides, by the model-theoretic Definition
\ref{Def_Progression}, progression is not uniquely defined -- there
is no restriction on occurrences of the unique name axioms
in progression, which may easily lead to loss of
inseparability of the components. In other words, progression may
logically imply unique name axioms even if the
initial theory did not imply them. Some decomposition components
of progression may imply such formulas, while the others may not.
For this reason, we speak of inseparability ``modulo''
theory $\DD_{una}$ in the first point of the theorem below. In particular, we have to
make the assumption that not only the components $\{D_i\}_{i\in
I\subseteq\omega}$ of the initial theory are pairwise
$\D$--inseparable, but so are the theories $\{\DD_{una}\cup
D_i\}_{i\in I}$. 

For fluent-free \textit{deltas}, the progression entails exactly those $\D$-formulas, which are entailed already by the initial theory (together with UNA-axioms), and the question is how these formulas can be ``distributed'' between the components. The second point of the theorem rules out the case (described in Example \ref{Ex_SplitOfComponent}), when $\D$-consequences are split between the components of progression. Note that in the theorem we do not specify how the progression was obtained (cf. Theorem
\ref{Teo_PreserveCompLocalEffect}) and the only condition that
relates the components of progression with those of the initial theory
says about containment of $\D$--consequences. 

\begin{teo}[Preservation of $\D$-insep. for fluent-free $\D$]\label{Teo_PreservationOfInseparability}
Let $\LL$ have PIP and $\DD$ be a $\mathcal{BAT}$ in which
$\DD_{S_0}$ and $\DD_{una}$ are theories in $\LL$. Let
$\si\subseteq\sig(\DD_{S_0})$ be a fluent--free signature and
denote $\D=\sig(\DD_{una})\cup\si$. Suppose the following:
\begin{itemize}
\item $\DD_{S_0}$ is $\si$--decomposable with some components
$\{D_i\}_{i\in I\subseteq\omega}$ such that the theories from
$\{\DD_{una}\cup D_i\}_{i\in I}$ are pairwise $\D$--inseparable;

\item $\DD_{S_\alpha}(S_0/S_\alpha)$ is equivalent to the union of
theories $\{D'_j\}_{j\in J\subseteq\omega}$ such that for every
$j\in J$ and some $i\in I$, $ \ \Cons(\DD_{una}\cup
D'_j,\D)\supseteq \Cons(\DD_{una}\cup D_i,\D)$.
\end{itemize}
Then the theories from $\{\DD_{una}\cup D'_j\}_{j\in
J\subseteq\omega}$ are pairwise $\D$--inseparable.
\end{teo}

\begin{proof} 
Let us demonstrate that for all $j\in J$ we have
 $\Cons(\DD_{una}\cup D'_j,\D)=\Cons(\DD_{una}\cup\DD_{S_0},\D)$, from which the
statement of the theorem obviously follows. Essentially, we prove
the following inclusions (the corresponding points of the proof
are marked with circles):

\begin{center}

\makebox[12cm]{

{\footnotesize $\Cons(\DD_{una}\cup\DD_{S_\alpha}(S_0/S_{\alpha}),  \D) \ \subseteq \ \Cons(\DD_{una}\cup\DD_{S_\alpha}, \D) \ \subseteq \ \Cons(\DD_{una}\cup\DD_{S_0},  \D)$}

\put(-95,15){\circle{11}} \put(-97.5,12){2}

\put(-202,15){\circle{11}} \put(-204.5,12){3}

\put(-20,-17){\circle{11}} \put(-22.3,-19.9){1}

\put(-40,-20){\rotatebox[origin=c]{270}{$\subseteq$}}

\put(-155,-38){$\supseteq$}

\put(-152,-48){\circle{11}} \put(-154.8,-51){4}

\put(-182,-25){\framebox{{\tiny Theorem conditions}}}

\put(-286,-20){\rotatebox[origin=c]{90}{$\subseteq$}}

\put(-298,-17){\circle{11}} \put(-300.8,-19.8){4}

\put(-272,-20){\framebox{{\tiny Theorem conditions}}}

\put(-85,-38){{\footnotesize $\Cons(\DD_{una}\cup\DD_i,\D)$}}

\put(-320,-38){{\footnotesize $\Cons(\DD_{una}\cup\DD_j',\D)$}}

}
\end{center}

\medskip

1) Note that for any $i\in I$, $\DD_{S_0}$ is $\si$--decomposable
with components $D_i$ and $\bigcup_{k\in I\setminus\{i\}}D_k$. We
claim that $\DD_{una}\cup D_i$ and $\DD_{una}\cup\bigcup_{k\in
I\setminus\{i\}}D_k$ are $\D$--inseparable. Let $\varphi$ be a
formula in signature $\D$. If $\DD_{una}\cup D_i\models\varphi$
then clearly, $\DD_{una}\cup\bigcup_{k\in
I\setminus\{i\}}D_k\allowbreak\models\varphi$ by $\D$--inseparability from
the condition of the theorem. On the other hand, if
$\DD_{una}\cup\bigcup_{k\in I\setminus\{i\}}D_k\models\varphi$
then by PIP we have $\bt_{una}\cup\bigcup_{k\in
I\setminus\{i\}}\bt_k\models\varphi$, where
$\DD_{una}\models\bt_{una}$,
$\sig(\bt_{una})\subseteq\sig(\DD_{una})$ and $D_k\models\bt_k$
for $k\in I\setminus\{i\}$, $\sig(\bt_k)\subseteq\D$. Again, by
$\D$--inseparability, for each $k\in I\setminus\{i\}$ we have
$\DD_{una}\cup D_i\models\bt_k$ and thus, $\DD_{una}\cup
D_i\models\varphi$.

Therefore, if $\varphi\in\Cons(\DD_{una}\cup\DD_{S_0},\D)$, then
for every $i\in I$, $[\DD_{una}\cup\bigcup_{k\in
I\setminus\{i\}}D_k]\ \allowbreak \cup \ [\DD_{una}\cup D_i]\models\varphi$
and then by PIP and inseparability shown above, $\DD_{una}\cup
D_i\models\varphi$. Since $\DD_{S_0}\models\bigcup_{i\in I}D_i$ by
decomposability, we obtain $\Cons(\DD_{una}\cup
\DD_{S_0},\D)=\Cons(\DD_{una}\cup D_i,\D)$ for all $i\in I$.

\smallskip

2) Let us show that $\Cons(\DD_{una}\cup
\DD_{S_\alpha},\D)\subseteq\Cons(\DD_{una}\cup\DD_{S_0},\D)$.
First, take a formula $\psi\in\Cons(\DD_{una}\cup
\DD_{S_\alpha},\D)$, which does not contain situation terms. From
the definition of progression, every model of $\DD$ is a model of
$\DD_{una}\cup\DD_{S_\alpha}$, so
$\DD\models\DD_{una}\cup\DD_{S_\alpha}$ and hence,
$\DD\models\psi$. If $\DD_{una}\cup\DD_{S_0}\not\models\psi$,
then $\DD_{una}\cup\DD_{S_0}\cup\{\neg\psi\}$ is satisfiable and
since $\psi$ is a uniform formula, by Proposition
\ref{Prop_RelativeSatisfiability}, $\DD\cup\{\neg\psi\}$ is
satisfiable, which contradicts $\DD\models\psi$. Therefore,
$\DD_{una}\cup\DD_{S_0}\models\psi$.

It remains to verify that the set $\Cons(\DD_{una}\cup
\DD_{S_\alpha},\D)$ is axiomatized by sentences which do not
contain situation terms. We have
$\D=\sig(\DD_{una})\cup\si\subseteq\sig(\DD_{una})\cup\sig(\DD_{S_0})$,
so $\{do,\preceq,Poss\}\cap\D=\varnothing$, by definition of
$\DD_{una}$ and $\DD_{S_0}$. As $\si$ if fluent-free by the
condition of the theorem (and $\sig(\DD_{una})$ is fluent-free by
definition of $\cal{BAT}$), $\D$ may contain only
situation--independent predicates and functions. Thus, any formula
$\varphi\in\Cons(\DD_{una}\cup \DD_{S_\alpha},\D)$ may contain
situation terms only in equalities, where each term is either the
constant $S_0$ (in case $S_0\in\si$) or a bound variable of sort
situation. Suppose that this is the case and there is no
$\psi\in\Cons(\DD_{una}\cup \DD_{S_\alpha},\D)$ such that
$\psi\models\varphi$ and $\psi$ does not contain situation terms.
By the syntax of $\LL_{sc}$ and the choice of $\D$, then $\varphi$
is a boolean combination of formulas without situation terms and
sentences over signature $\{S_0\}$ stating that $\varphi$ has a
model with cardinality $|Sit|$ of sort \textit{situation} lying in
the interval $[n,m]$ for $n\in\omega$ and
$m\in\omega\cup\{\infty\}$. We denote sentences of this form by
$\exists^{[n,m]}\theta_=$. We may assume that $\varphi$ is in
conjunctive normal form and that there is a formula $\xi$, a boolean
combination of $\exists^{[n,m]}\theta_=$ such that
$\not\models\xi$, $\not\models\neg\xi$, either $\xi$ or
$\xi\vee\eta$ is a conjunct of $\varphi$, and
$\eta\not\in\Cons(\DD_{una}\cup \DD_{S_\alpha},\D)$,
$\sig(\eta)\subseteq\D$, is a formula without situation terms. As
$\not\models\xi$ and $\not\models\neg\xi$, there are
$n,m\in\omega$ such that $\xi$ does not have a model with
$|Sit|=n$ and $\neg\xi$ does not have a model with $|Sit|=m$. Then
by Lemma \ref{Lemma_ModelPropertyofUniformTheories}, we conclude
that $\DD_{una}\cup \DD_{S_\alpha}\not\models\xi$ and
$\DD_{una}\cup \DD_{S_\alpha}\not\models\neg\xi$. In particular,
$\xi$ can not be a conjunct of $\varphi$. If $\xi\vee\eta$ is a
conjunct, then there exists a model $\M$ of $\DD_{una}\cup
\DD_{S_\alpha}$ such that $\M\models\xi$ and $\M\not\models\eta$.
Then, by applying Lemma \ref{Lemma_ModelPropertyofUniformTheories}
again, there must be a model $\M'$ of $\DD_{una}\cup
\DD_{S_\alpha}$ with $|Sit|=n$ where the interpretation of
situation--independent predicates and functions is the same as in
$\M$. Thus, $\M'\not\models \xi$ and since $\eta$ does not contain
situation terms, $\M'\not\models \eta$, which contradicts
$\DD_{una}\cup \DD_{S_\alpha}\models\varphi$.

\smallskip

3) Now let us demonstrate that
$\Cons(\DD_{una}\cup\DD_{S_\alpha}(S_0/S_\alpha),\D)\subseteq\Cons(\DD_{una}\cup\DD_{S_\alpha},\D)$.
Note that $\DD_{una}\cup\DD_{S_\alpha}(S_0/S_\alpha)$ is uniform
in $S_0$. Following the above proved, assume that there is a
formula
$\varphi\in\Cons(\DD_{una}\cup\DD_{S_\alpha}(S_0/S_\alpha),\D)$
such that $\varphi$ does not contain situation terms and
$\DD_{una}\cup\DD_{S_\alpha}\not\models\varphi$. Take a model $\M$
of $\DD_{una}\cup\DD_{S_\alpha}$ such that $\M\not\models\varphi$.
Then, by Lemma \ref{Lemma_ModelPropertyofUniformTheories}, there
exists a model $\M'$ of $\DD_{una}\cup\DD_{S_\alpha}$ such that
the domain for sort \textit{situation} in $\M'$ is a singleton set
(i.e., the interpretation of terms $S_0$ and $S_\alpha$ coincide in
$\M'$) and the interpretation of situation--independent symbols is
the same in $\M$ and $\M'$. Then $\M'\not\models\varphi$, but
clearly $\M'\models\DD_{una}\cup\DD_{S_\alpha}(S_0/S_\alpha)$
which contradicts the assumption
$\DD_{una}\cup\DD_{S_\alpha}\not\models\varphi$.

\smallskip

4) Finally, by the condition of the theorem, for all $j\in J$, we
have $\DD_{una}\cup D'_j\subseteq\DD_{una}\cup
\DD_{S_\alpha}$ $(S_0/S_\alpha)$ and from points 1--3 above we
obtain
$\Cons(\DD_{una}\cup\DD_{S_\alpha}(S_0/S_\alpha),\D)\subseteq\Cons(\DD_{una}\cup
\DD_{S_0},\D)$. Hence, for all $j\in J$ we have
$\Cons(\DD_{una}\allowbreak\cup D'_j,\D)\subseteq\Cons(\DD_{una}\cup
\DD_{S_0},\D)$. On the other hand, we also have
$\Cons(\DD_{una}\cup D_i,\D)\subseteq\Cons(\DD_{una}\cup D'_j,\D)$
from the condition of the theorem. Therefore from the inclusion
$\forall \ i\in I$
$\Cons(\DD_{una}\cup\DD_{S_0},\D)\subseteq\Cons(\DD_{una}\cup D_i,\D)$ of
point 1 we conclude that $\Cons(\DD_{una}\cup
D'_j,\D)=\Cons(\DD_{una}\cup \DD_{S_0},\D)$ for all $j\in J$. 
\end{proof}


The next theorem provides a result on local-effect ${\cal{BAT}}s$
with initial theories in first-order logic for which progression
becomes more concrete, since it can be computed by syntactic
manipulations. In contrast to Theorem
\ref{Teo_PreservationOfInseparability}, this allows us to judge about
inseparability without the theory $\DD_{una}$ in background.
Recall that, in general, a $\cal{BAT}$ includes non-trivial precondition axioms.
On the right-hand side of each precondition axiom, there is a formula $\Pi_A(\bar{x},s)$ 
that is a formula uniform in $s$ with free variables among $\bar{x}$ and $s$. 
However, any $\cal{BAT}$ can be transformed into an action theory without precondition axioms by introducing the right hand side formulas
from the precondition axioms as conjuncts of context conditions for each corresponding active
position of an action term in a SSA.
Therefore, without loss of generality, and for simplicity of presentation, we subsequently consider 
 the $\cal{BAT}$s where  all precondition axioms are trivial.

Essentially, the conditions of the theorem are defined to
guarantee componentwise computation of progression for a
decomposable initial theory. A finite set $\DD_{ss}$ of the SSAs is considered to be syntactically divided into the union of $|I|$ sub-theories sharing some fluent-free signature $\D_1$ (which may include actions, static predicates, and object constants), as well as function $do$ (which occurs in every SSA). Informally, each of $|I|$ sub-theories is about a different set of properties, e.g., one of them could be about the blocks world, 	while another could be about the logistics world, with the two theories 	possibly sharing some constants, such as a box name, and
	situation-independent predicates, such as shapes of the boxes. The initial theory $\DD_{S_0}$ is $\D_2$--decomposable, for a fluent-free signature $\D_2$, into $|J|$ components. To distinguish visually components $D'_j$ of $\DD_{S_0}$ from the components $D_i$ of $\DD_{ss}$, we write $D'_j$ with apostrophe when we mean
	components of  $D_{S_0}$, and $D_i$ without apostrophe when we mean
	groups of SSAs. Informally, each component $D'_j$ is about a
	separate aspect of the initial theory. The syntactic form of
	the initial theory may not reveal the components readily, but
	they can be discovered through decomposition. Naturally, it is
	expected that independent components of the initial theory
	should remain independent after doing any actions. This imposes
	a condition that each component from the initial theory should be
	related with its own group of SSAs.
	
	The last two conditions of the theorem enforce that the subtheories of $\DD_{ss}$ are aligned with the components of $\DD_{S_0}$ via syntactic occurrences of fluents. The second to last condition says that every fluent mentioned in a SSA must also occur in the initial theory $\DD_{S_0}$. It is easy to satisfy by adding tautologies with the corresponding fluents to $\DD_{S_0}$. Together with the last condition it guarantees that for every SSA $\varphi$ containing fluents $F_1,\ldots,F_n$ there is a corresponding component of $\DD_{S_0}$, which describes the initial interpretation of these fluents, and this is the component that must be updated upon executing an action mentioned in active position of $\varphi$. The last condition also enforces that actions, static predicates, or object constants separated by the decomposition of $\DD_{S_0}$ must be also separated by the subtheories of $\DD_{ss}$, whenever they occur in SSAs. This guarantees that these symbols do not become connected after progression (as opposed to the situation presented in Example \ref{Ex_DecompLostUnderProgression}). Thus, the theory $\DD_{ss}\cup\DD_{S_0}$ can be divided into parts (consisting of successor-state axioms and statements about the initial situation) which may mention common actions, static predicates and constants, but talk about different fluents. In other words, these subtheories define independent sets of situation-related properties, which is natural for a composite action theory describing a single domain of objects from a number of different perspectives. \commentout{For instance, one can consider an action theory describing the blocks world, in which a block can be moved on top of another one, while changing its color. Then there are two perspectives from which one can reason about the collection of blocks, with the first one being the block positions and the second one being their colors.}\commentout{The considered theory $\DD_{ss}$ of a $\cal{BAT}$ is syntactically divided into the union of theories sharing some fluent-free signature $\D_1$, the initial theory $\DD_{S_0}$ is $\D_2$--decomposable for a fluent-free signature $\D_2$, and the subtheories of $\DD_{ss}$ are aligned with the components of $\DD_{S_0}$ via syntactical occurrence of fluents.}Note that it is allowed for a single action to have effects on groups of fluents (possibly, all the fluents at once, without regard to distribution of the fluents between the subtheories), which is reflected in the theorem condition that $\Delta_1$ is just fluent-free, but not action free. We impose stronger restriction in Corollary \ref{Cor_StrongPreservationOfComponents}, which describes a class of $\mathcal{BAT}s$ representing composite subject domains, like the one mentioned in the running example from Section \ref{Sect_SituationCalculus}.
	
For the reader's convenience, we
stress that in the formulation of the theorem, the indices $i$ and
$j$ vary over components of $\DD_{ss}$ and $\DD_{S_0}$,
respectively. The signatures $\D_1$ and $\D_2$ are the sets of
allowed common symbols between the components of $\DD_{ss}$ and
$\DD_{S_0}$, respectively. We recall that $\A$ ($\F$, respectively) denotes the set of action functions (the set of fluents, respectively) from the alphabet of the language of the situation calculus.

\begin{teo}[Preservation of components in local-effect $\cal{BAT}$]\label{Teo_PreserveCompLocalEffect}
Let $\DD$ be a local-effect $\cal{BAT}$, with $\DD_{S_0}$ an initial
theory in first-order logic. Let $\D_1$, $\D_2$ be fluent-free
signatures, $do\not\in\D_1,\D_2$, and $\alpha=A(\bar{c})$ a ground
action term. Denote $\D=\D_1\cup\D_2\cup\{c_1,\ldots ,c_k\}$, if
$\bar{c}=\langle c_1,\ldots ,c_k \rangle$, and
suppose the following:
\begin{itemize}
\item $\DD_{ss}$ is the union of theories $\{D_i\}_{i\in I}$, 
with $\sig(D_n)\cap\sig(D_m)\subseteq\D_1\cup\{do\}$ for all $n,m\in
I\neq\varnothing$, $n\neq m$;

\item $\DD_{S_0}$ is $\D_2$--decomposable into finite components
$\{D_j'\}_{j\in J}$ uniform in $S_0$ such that $\sig(D_j')\setminus\D\neq\varnothing$, for all $j\in J$;  

\item $\sig(\DD_{ss})\cap\F\subseteq\sig(\DD_{S_0})$;

\item for every $i\in I$, there is $j\in J$ such that
$\sig(D_i)\cap\sig(\DD_{S_0})\subseteq\sig(D_j')$.
\end{itemize}

Then $\DD_{S_\alpha}(S_0/S_\alpha)$ is $\D$--decomposable. If the
components $\{D_j'\}_{j\in J}$ are pairwise $\D$--inseparable,
then so are the components of $\DD_{S_\alpha}\!(S_0/S_\alpha)$ in
the corresponding decomposition.
\end{teo}

\begin{proof} 
The proof consists of two parts, both of which rely on the constructive definition of progression for local-effect $\mathcal{BAT}$s from Section \ref{Section_Progression} and component properties of forgetting discussed in Section \ref{Section_Forgetting}. In the first part, we show $\Delta$-decomposability of $D_{S_\alpha}(S_0/S_\alpha)$ by constructing its components explicitly and in the second part we prove that these components are $\Delta$-inseparable. 

1) By definition of $\cal{BAT}$, for every $i\in I$, we have
$\sig(D_i)\cap\F\neq\varnothing$ and thus, from the conditions of
the theorem, $\sig(D_i)\cap\sig(\DD_{S_0})\neq\varnothing$,
$\sig(\DD_{ss})\cap\F = \sig(\DD_{S_0})\cap\F$. Hence, for every
$i\in I$ there is $j\in J$ such that
$\sig(D_i)\cap\F\subseteq\sig(D_j')$. Moreover, such $j\in J$ is
unique for every $i\in I$, because otherwise there would exist
$n,m\in J$, $n\neq m$, such that $\sig(D_{n}')\cap
\sig(D_{m}')\cap\F\neq\varnothing$, which contradicts the
condition that $\D_2$ is fluent-free. Therefore, there is a map
$f: I\rightarrow J$ such that for every $i\in I$,
$\sig(D_i)\cap\F\subseteq\sig(D_{f(i)}')$. Note that there may
exist $j\in J$ such that $\sig(D_j')\cap\F=\varnothing$ and in
this case $j$ is the image of no $i\in I$. Let us denote the image
of $f$ by $\tilde{J}$ (so, $\tilde{J}\subseteq J$). 

Now, for every $i\in I$, consider the set of formulas $D_i[\Omega]$,
the instantiation of $D_i$ w.r.t. $\Omega(S_0)$, and for each
$j\in\tilde{J}$, denote $\widetilde{D}_j=[ \ \bigcup_{i\in
f^{-1}(j)} (D_i[\Omega]) \ ]\cup D_{j}'$. Then, by Proposition
\ref{Prop_ProgressionForLocalEffectCase}, $\DD_{S_\alpha}(S_0/S_\alpha)$ (progression
 of $\DD_{S_0}$ wrt $\alpha$, with
term $S_\alpha$ substituted with $S_0$) is logically
equivalent to
$$[ \ \forget(\bigcup_{j\in\tilde{J}}\widetilde{D}_j ,\Omega(S_0)) \ \cup\bigcup_{j\in J\setminus\tilde{J}}
D_j' \ ] \ (S_0/S_\alpha).$$ As $\D_1$ and $\D_2$ are fluent-free,
the signatures $\{\sig(\widetilde{D}_j)\}_{j\in\tilde{J}}$ do not
have fluents in common and thus, by Corollary
\ref{Cor_ForgettingWithSpecComponent}, $\DD_{S_\alpha}(S_0/S_\alpha)$ is
equivalent to
$$[ \ \bigcup_{j\in\tilde{J}} \forget(\widetilde{D}_j
,\Omega(S_0)\mid_j) \ \cup\bigcup_{j\in J\setminus\tilde{J}} D_j'
\ ] \ (S_0/S_\alpha),$$ where for $j\in\tilde{J}$,
$\Omega(S_0)\mid_j$ is the subset of ground atoms from
$\Omega(S_0)$ with fluents from $\sig(D_j')$. For all $j\in
J\setminus\tilde{J}$, we have $\sig(D_j')\cap\F=\varnothing$ and
$D_j'$ is uniform in $S_0$, so it follows that
$S_0\not\in\sig(D_j')$ and thus, $\DD_{S_\alpha}(S_0/S_\alpha)$ is
equivalent to the union
$$[ \ \bigcup_{j\in\tilde{J}} \forget(\widetilde{D}_j
,\Omega(S_0)\mid_j) \ ] \ (S_0/S_\alpha) \ \cup\bigcup_{j\in
J\setminus\tilde{J}} D_j'.$$ For every $j\in J$, let $D_j''$ be
the set of formulas $(\forget(\widetilde{D}_j
,\Omega(S_0)\mid_j))(S_0/ S_\alpha)$ (in case $j\in\tilde{J}$) or
the set of formulas $D_j'$ (if $j\in J\setminus\tilde{J}$). So
$\DD_{S_\alpha}(S_0/S_\alpha)$ is equivalent to $\bigcup_{j\in
{J}}D_j''$. By the definition of forgetting a set of
ground atoms one can assume that $\sig(D_j')\subseteq\sig(D_j'')$ and $\sig(D_j'')\setminus\sig(D_j')\subseteq\sig(\DD_{ss})\cup\{c_1,\ldots, c_k\}$, for all $j\in J$. 

Let us show that $[\sig(D_i'')\cap\sig(D_j'')]\subseteq\D$, for all distinct $i,j\in J$. Assume there are distinct $i,j\in J$ such that $[\sig(D_i'')\cap\sig(D_j'')]\setminus\D=\Sigma\neq\varnothing$, for a signature $\Sigma$. Then $do\not\in\Sigma$, since both $D_i''$ and $D_j''$ are uniform in $S_0$. If there is a single subtheory $D_m$ of $\DD_{ss}$, $m\in I$, such that $\Sigma\subseteq\sig(D_m)$, then the last two conditions of the theorem yield $\Sigma\subseteq\D_2$, which is a contradiction, because we have assumed $\Sigma\cap\D=\varnothing$. If there are distinct subtheories $D_m$ and $D_n$ of $\DD_{ss}$, $m,n\in I$, such that $\Sigma\subseteq\sig(D_m)\cap\sig(D_n)$, then $\Sigma\subseteq\D_1$, and we again arrive at contradiction.

It follows that the pairwise intersection of any signatures from
$\{\sig(D_j'')\}_{j\in J}$ is a subset of $\D$ and it follows from the second condition of the theorem that $\sig(D_j'')\setminus\D\neq\varnothing$. Then $\{D_j''\cup
Taut(\D,j)\}_{j\in J}$ is $\D$--decomposition of
$\DD_{S_\alpha}(S_0/S_\alpha)$, where for each $j\in J$,
$Taut(\D,j)$ is a set of tautologies in signature
$\D\setminus\sig(D_j'')$ which are uniform in $S_0$.

\smallskip

2) Now let us verify that the sets of formulas from
$\{D_j''\}_{j\in J}$ are pairwise $\D$--insepa\-rable, if so are the
components of $\DD_{S_0}$.

a) First, consider the sets from the union
$$\bigcup_{j\in\tilde{J}}\widetilde{D}_j \ \cup\bigcup_{j\in
J\setminus\tilde{J}} D_j'. \phantom{abcde} (\ddag)$$ The pairwise
intersection of their signatures is contained in
$\D\cup\sig(S_\alpha)$. We claim that the sets from this union are
pairwise $\D$--inseparable.

By our definition, for all $j\in \tilde{J}$ we have
$D_j'\subseteq\widetilde{D}_j$ and hence,
$\Cons(D_j',\D)\subseteq\Cons(\widetilde{D}_j,\D)$, so let us
check that $\Cons(\widetilde{D}_j,\D)\subseteq\Cons(D_j',\D)$ for
every $j\in \tilde{J}$. Each formula in $D_i[\Omega]$, for $i\in
f^{-1}(j)$, $j\in \tilde{J}$, has the form
$$F(\bar{c},do(A(c_1,\ldots
,c_k),S_0))\leftrightarrow (\varepsilon_1\wedge\phi^+) \ \vee \
(F(\bar{c},S_0)\wedge\varepsilon_2\wedge\phi^{-}), \phantom{abc}
(\ast)$$ where $F$ is a fluent from $\sig(D_j')$, $\bar{c}$ is a
vector of constants from $\{c_1,\ldots ,c_k\}$, $\phi^+$, $\phi^-$
are sentences uniform in $S_0$, and each $\varepsilon_1$,
$\varepsilon_2$ equals \textit{true} or \textit{false} (the
parameters to summarize different cases of this formula). This is
a definition of ground atom $F(\bar{c},do(A(c_1,\ldots ,c_k),S_0)$
via fluents at situation $S_0$ and situation-independent
predicates and functions. Therefore, since $\D$ is fluent-free and
for all $j\in\tilde{J}$, $D_j'$ is uniform in $S_0$, every model
$\M$ of $D_j'$ can be transformed into a model $\M'$ of
$\widetilde{D}_j$ which agrees with $\M$ on $\D$. The model $\M'$
is obtained in two steps. First, we expand $\M$ with an arbitrary
interpretation of function $do$ and situation-independent
predicates and functions from
$\sig(D_i[\Omega])\setminus\sig(D_j')$ for every $i\in f^{-1}(j)$.
Then we continue with this expanded model and modify the truth
value of each fluent $F$ at the interpretation of the tuple
$\langle\bar{c}, do(A(c_1,\ldots ,c_k),S_0)\rangle$ according to
the obtained truth value of the formula in the definition of
$F(\bar{c},do(A(c_1,\ldots ,c_k),S_0)$ above. This gives us the
model $\M'$. Hence, if $\varphi\in\Cons(\widetilde{D}_j,\D)$ and
$\varphi\not\in\Cons(D_j',\D)$, then there is a model $\M$ of
$D_j'$ such that $\M\not\models\varphi$, but then
$\M'\models\widetilde{D}_j$ and $\M'\not\models\varphi$, a
contradiction. Therefore, we conclude that for all
$j\in\tilde{J}$, $\Cons(\widetilde{D}_j,\D)=\Cons(D_j',\D)$ and,
by pairwise $\D$--inseparability of the components of $\DD_{S_0}$,
the sets from the union $(\ddag)$ are $\D$--inseparable.

\smallskip

b) Since $\D$ is fluent-free and $\Omega(S_0)$ consists only of
ground atoms with fluents, from Corollary
\ref{Cor_ForgettingWithSpecComponent} we conclude that the sets
from the following union are $\D$--insepa\-rable:
$$\bigcup_{j\in\tilde{J}} \forget(\widetilde{D}_j
,\Omega(S_0)\mid_j) \ \cup\bigcup_{j\in J\setminus\tilde{J}}
D_j'.$$

Now we are ready to prove that the sets from $\{D_j''\}_{j\in J}$
are pairwise $\D$--inseparable. For every $j\in \tilde{J}$, let us
denote $G_j=\forget(\widetilde{D}_j ,\Omega(S_0)\mid_j)$. We will
demonstrate that for every $j\in \tilde{J}$ it holds
$\Cons(G_j(S_0/S_\alpha),\D)=\Cons(G_j, \ \D)$, from which the
statement follows. First, let us verify that $\Cons(G_j$
$(S_0/S_\alpha),\D)\subseteq\Cons(G_j, \ \D)$. Assume that for
some $j\in\tilde{J}$ (we fix this $j$ for the following) there is
a formula $\varphi\in\Cons(G_j(S_0/S_\alpha),\D)$ and a model $\M$
of $G_j$ such that $\M\not\models\varphi$, and arrive at
contradiction.

By the syntactic definition of forgetting a ground atom, the term
$S_\alpha$ occurs in $G_j$ only in subformulas obtained from the
definitions $(\ast)$, so let us consider such a definition for a
ground atom $F(\bar{c}, S_\alpha)$ with some fluent $F$. Let us
recall that $G_j$ is the result of forgetting a set of ground
atoms with fluents having $S_0$ as situation argument.
Since $\bar{c}$ is the vector of object arguments in the
definition of $F(\bar{c}, S_\alpha)$ in $(\ast)$, we have
$F(\bar{c}, S_0)\in\Omega(S_0)\mid_j$. 
Therefore, if $\M\models \ ^\varepsilon F(\bar{c}, S_0)$
($\varepsilon$ denotes the optional negation in front of atom),
then there is a model $\M'\models\neg^\varepsilon F(\bar{c}, S_0)$
such that $\M'\sim_\si\M$, with $\si=F(\bar{c}, S_0)$, and hence,
$\M'\not\models\varphi$ (since $\D$ is fluent--free) and the truth
value of $F(\bar{c}, S_\alpha)$ in $\M $ and $\M'$ is the same.
Hence, either in $\M$ or $\M'$ the truth values of $F(\bar{c},
S_\alpha)$ and $F(\bar{c}, S_0)$ coincide. The similar argument
applies to the whole set of definitions $(\ast)$ from
$\widetilde{D}_j$ under forgetting the set $\Omega(S_0)\mid_j$.
Therefore we may assume that in $\M$ or $\M'$, for each fluent
$F\in\sig(G_j)$ the values of $F(\bar{c},S_\alpha)$ and
$F(\bar{c},S_0)$ coincide. So $\M\models G_j(S_0/S_\alpha)$ or
$\M'\models G_j(S_0/S_\alpha)$ which is a contradiction, because
$\varphi$ holds in neither of these models.

To prove the reverse inclusion $\Cons(G_j, \ \D)\subseteq\Cons(G_j
(S_0/S_\alpha),\D)$, observe that $G_j (S_0/S_\alpha)$ is uniform
in $S_0$. Hence, by an observation similar to Lemma
\ref{Lemma_ModelPropertyofUniformTheories}, every model $\M$ of
$G_j (S_0/S_\alpha)$ can be expanded to a model $\M'$, where the
interpretation of function $do$ is such that the values of terms
$S_\alpha$ and $S_0$ in $\M'$ coincide. Then $\M'\models G_j$ and
thus, there is no formula $\varphi\in\Cons(G_j, \ \D)$ such that
$\varphi\not\in\Cons(G_j(S_0/S_\alpha), \ \D)$. 
\end{proof}


We note that a result similar to Theorem
\ref{Teo_PreserveCompLocalEffect} can be proved in a more general
case, for progression of not-necessarily local-effect
$\mathcal{BAT}$s, by considering progression as a set of
consequences of $\DD_{una}\cup\DD_{ss}\cup\DD_{S_0}$ uniform in
$S_\alpha$ or using the second-order definition of progression from Theorem 2.10 in \cite{ProgressionLocalEffect}. Since both definitions of progression are non-constructive, one would have to deal with background theories such as $\DD_{una}$, when reasoning about decomposition of the initial theory. Although it would be possible to define a more general notion of decomposability wrt a background theory by following this direction, this study would take us too far away from the goals of this paper, and it would not be illuminating.

The proof of the theorem uses Proposition \ref{Prop_ProgressionForLocalEffectCase} and the component properties of forgetting from Section \ref{Section_Forgetting}. The important observation behind this result is that in order to compute progression of an initial theory wrt an action having effects only on fluents from one 
decomposition component, it suffices to compute forgetting only in this 
component. Given a decomposition of the initial theory into inseparable components, 
the rest of the conditions in the theorem are purely syntactical and easy to 
check. For example, these conditions would naturally hold if one merges weakly-related action theories, as illustrated in the running example (continued below). 
SSAs can be grouped into $|I|$ components
by drawing a graph with fluent names as vertices, and an edge from the fluent
on the left-hand side of each SSA going to each fluent occurring on 
the right-hand side of the same SSA. Similarly, it is easy to check 
the last condition of the Theorem that guarantees alignment of groups of
axioms in SSAs with decomposition components of $\DD_{S_0}$.

In the above conditions, observe that if an action
$A$ occurs in active position of SSAs from two different
sub-theories of $\DD_{ss}$, then computing progression may
involve forgetting in two corresponding components of $\DD_{S_0}$.
This can potentially lead to appearance of new common $\D_1$--symbols in the
components of progression. As a consequence, $\D_2$--decomposability of progression
may be destroyed, but it is desirable to preserve it.  A practically important 
class of $\cal{BAT}$s, for which this interference can be avoided, is described
in the corollary below. Note the first condition in the corollary that 
every action mentioned in $\cal{BAT}$ can have
effects on fluents only from one component of $\DD_{ss}$. Together with the second condition this guarantees preservation of $\D_2$-decomposability and inseparability of the initial theory after progression. The third condition in the corollary guarantees preservation of all the conditions of Theorem \ref{Teo_PreserveCompLocalEffect} for the $\cal{BAT}$ obtained after progression and thus, one can compute progression for arbitrary long sequences of actions while preserving decomposability of $\DD_{S_\alpha}(S_0/S_\alpha)$ and inseparability of its components. 

\begin{cor}[Strong preservation of components in local-effect
$\cal{BAT}$s]\label{Cor_StrongPreservationOfComponents} In the conditions and
notations of Theorem \ref{Teo_PreserveCompLocalEffect}, let $\alpha=A(\bar{c})$ be a
ground action term, where $\bar{c}\!=\!\langle c_1,\ldots,c_k\rangle$ is a tuple of constants, and let the following conditions hold:\vspace{-0.2cm}
\begin{itemize}
\item no action function is in $\D_1$;


\item $\{c_1,\ldots,c_k\}\subseteq\sig(D_j')$, for some $j\in J$,\newline whenever $A$ is in active position in a SSA for a fluent $F\in\sig(D_j')$;

\item it holds that $\D_1\subseteq\D_2$.
\end{itemize}

Then $\DD_{S_\alpha}(S_0/S_\alpha)$ is $\D_2$--decomposable into
$\D_2$--inseparable components. Moreover, all the conditions of Theorem \ref{Teo_PreserveCompLocalEffect} hold for the $\cal{BAT}$ with the initial theory $\DD_{S_\alpha}(S_0/S_\alpha)$ obtained after progression.
\end{cor}

\begin{proof} 
\commentout{ Since all context conditions are tautological,
every $SSA$ is equivalent to a formula of the form

$\forall \ \bar{x},a,s \ [F(\bar{x},do(a,s))\leftrightarrow
\bigvee_{n\in N\subseteq\omega} ([\exists \bar{t}_n] \
a=A_k(\bar{t_n})) \vee \ F(\bar{x},s)\wedge$

\hspace{6cm} $\wedge\neg \bigvee_{m\in M\subseteq\omega} ([\exists
\bar{t}_m'] \  a=B_m(\bar{t}_m')).$

Therefore, since no action is in $\D_1$ and, by the condition of
Theorem \ref{Teo_PreserveCompLocalEffect}, $\D_1$ is fluent-free,
we conclude that $\D_1=\{do\}$. In the progression
$\DD_{S_\alpha}$, function $do$ can occur only in the term
$S_\alpha$ and hence,
$do\not\in\sig(\DD_{S_\alpha}(S_0/S_\alpha))$. On the other hand,
the constants $\{c_1,\ldots,c_k\}$ may occur in the signature of
progression.} By the first condition, action $A$ can be in active
position of SSAs of a single subtheory $D_i$ of $\DD_{ss}$. Then,
due to the componentwise computation of progression shown in the
proof of Theorem \ref{Teo_PreserveCompLocalEffect}, progression
can affect the single corresponding component $D_{f(i)}'$ of
$\DD_{S_0}$. The second condition of the
corollary guarantees that
$\{c_1,\ldots,c_k\}\subseteq\sig(D_{f(i)}')$ and together with the third condition this yields that $\DD_{S_\alpha}(S_0/S_\alpha)$ is $\D_2$--decomposable into $\D_2$--inseparable components, just like $\DD_{S_0}$ is.  

Computing the progression of $\DD_{S_0}$ wrt $\alpha$ is essentially a
syntactic modification of $D_{f(i)}'$ which may introduce signature
symbols from context conditions of $D_i$ only
into $D_{f(i)}'$ and into no other components of $\DD_{S_0}$. Denote by $D_{f(i)}''$ the theory obtained from $D_{f(i)}'$ in this way.

Let us verify that all the conditions of Theorem \ref{Teo_PreserveCompLocalEffect} are preserved for the $\cal{BAT}$ with the initial theory $\DD_{S_\alpha}(S_0/S_\alpha)$ obtained after progression. The first condition of the theorem holds by default. By the definition of forgetting ground atoms, one can assume that $\sig(D_{f(i)}')\subseteq\sig(D_{f(i)}'')$. Since $\DD_{S_\alpha}(S_0/S_\alpha)$ is $\D_2$--decomposable and, by the definition of $\DD_{S_\alpha}(S_0/S_\alpha)$, all the components of $\DD_{S_0}$ except $D_{f(i)}'$ remain unchanged after progression, the second and third conditions of the theorem hold. To show the last condition suppose the opposite, i.e. there is $k\in I$, for which the condition does not hold. Then $k\neq i$, since $\sig(D_{f(i)}')\subseteq\sig(D_{f(i)}'')$, and there is a subsignature $\Sigma\subseteq\sig(D_{f(i)}'')\setminus \sig(D_{f(i)}')$ such that $\Sigma\subseteq\sig(D_k)$. By the definition of $D_{f(i)}''$, we may assume that $\Sigma\subseteq\sig(D_i)$. As $\DD_{S_\alpha}(S_0/S_\alpha)$ is a set of formulas uniform in $S_0$, we have $do\not\in\Sigma$ and thus, $\Sigma\subseteq\D_1\subseteq\D_2$. Let $D_j'$ be the component of $\DD_{S_0}$, for which the condition $\sig(D_k)\cap\DD_{S_0}\subseteq\sig(D_j')$ holds. Since $\Sigma\subseteq\D_2$ and $\DD_{S_0}$ is $\D_2$--decomposable, we have $\Sigma\subseteq\sig(D_j')$ and thus $D_j'$ is the required component for $D_k$, a contradiction. 
\end{proof}

\medskip



\textbf{Example \ref{BWexample} (continuation)}.
Note that the $\cal{BAT}$ considered in the example satisfies the conditions of the corollary with fluent-free signatures $\D_1\!=\!\varnothing$ and $\D_2\!=\!\{Block,S_0\}$. The theory $\DD_{ss}$ is a union of two theories, with the intersection of signatures equal to $\{do\}$. As already noted in the example, the initial theory $\DD_{S_0}$ is $\D_2$--decomposable into $\D_2$--inseparable components.	
Now, consider the ground action $\alpha\!=\!move(A,B,C)$. By Corollary \ref{Cor_ForgettingWithSpecComponent} and Proposition \ref{Prop_ProgressionForLocalEffectCase}, in order to compute the theory $\DD_{S_\alpha}(S_0/S_\alpha)$ (the progression of $\DD_{S_0}$ wrt $\alpha$, with the term $S_\alpha$ substituted with $S_0$), it suffices to forget the ground atoms $On(A,B,S_0)$ and $Clear(C,S_0)$ in the first decomposition component of $\DD_{S_0}$ and update it with the ground atoms $On(A,C,S_0)$ and $Clear(B,S_0)$. The second component of $\DD_{S_0}$ remains unchanged. One can check that 
$\DD_{S_\alpha}(S_0/S_\alpha)$ is the union of the following theories:

\smallskip

\hspace{-0.2cm}\noindent
$
\begin{array}{l}
\varphi\land \psi \land (x\neq C) \rightarrow Clear(x,S_0)\\
\psi \rightarrow Block(x)\\
Block(B)\!\land Block(C)\!\land On(A,\!C,\!S_0)\!\land \neg On(A,\!B,\!S_0)\\
Clear(A,S_0)\land Clear(B,S_0)\land \neg Clear(C,S_0)
\end{array}
$

\smallskip\noindent 
and

\smallskip

\hspace{-0.2cm}\noindent
$
\begin{array}{l}	
(Top(x,S_0)\lor Inheap(x,S_0))\rightarrow\neg Block(x)\\
\exists x \ Block(x),
\end{array}
$

\medskip\noindent
where $\varphi$ and $\psi$, respectively, stand for

\medskip

\noindent
$
\begin{array}{l}
 (x\neq B) \land \neg\exists y \ ((y\neq A \lor x\neq B)\land On(y,x,S_0)),\\
 (x=A) \lor \exists y \ ((x\neq A \lor B\neq y)\land On(x,y,S_0)).
\end{array}
$

\medskip

The theory $\DD_{S_\alpha}(S_0/S_\alpha)$ is $\D_2$--decomposable by the syntactic form and there is no need to compute a decomposition again after progression. Corollary \ref{Cor_StrongPreservationOfComponents} guarantees that the obtained components are $\D_2$--inseparable. It is important that in this case we can compute progression for arbitrary long sequences of actions while preserving both decomposability of  $\DD_{S_\alpha}(S_0/S_\alpha)$ and inseparability of its components.

\section{Summary and Future Work}\label{Sect_Summary}
We have considered the impact of the theory update operations,
such as forgetting and progression on preserving the component
properties of theories, such as decomposability and inseparability.
Forgetting and progression have  a ``semantic nature'', since the input and the
output of these transformations are related to each other by using restrictions on the
classes of models. On the contrary, the decomposability and
inseparability properties are defined using entailment in a particular logic.
As logics (weaker than second-order) may not distinguish the needed classes
of models, the conceptual ``distance'' between these
two kinds of notions is potentially immense. This can be somewhat bridged
by the choice of either an appropriate logic, or appropriate
theories in the input. We have identified conditions that should be imposed
on the components of input theories to match these notions more closely.
Also, the Parallel Interpolation Property (PIP) was shown to be a relevant property
of logics in our investigations. The results can be briefly
summarized in the tables below. For brevity, we use $\si$ to
denote a signature or a ground atom. We slightly abuse notation and
consider $\si$ as a set of symbols even in the case of a ground
atom implying that in the latter case $\si$ consists of the single predicate
symbol from the atom. We assume that the input of operations
of forgetting and progression is a union of theories $\bt_1$ and
$\bt_2$, with $\sig(\bt_1)\cap\sig(\bt_2)=\D$, for a signature
$\D$.

\medskip

\footnotesize

\begin{center}

\begin{tabular}{|p{0.87in}|c|} \hline
\hspace{0.4cm} \textbf{Property} & \hspace{0.8cm}
\textbf{Condition}
\hspace{2.5cm} \textbf{Result} \hspace{1.7cm} \textbf{Reference}\\
\hline
\vspace{-0.7cm}
Preservation of $\D$--inseparability of $\bt_1$ and $\bt_2$ under
forgetting $\si$ &
\begin{tabular}{p{1.3in}|p{3cm}|c}\hline
$\si\cap\D=\varnothing$ & YES & Corollary
\ref{Cor_ForgettingWithSpecComponent}\\ \hline
$\si\subseteq\D$ and $\si$ is a ground atom & \vspace{0.07cm} NO & Example \ref{Ex_InseparabilityLostForgettingAtom}\\
\hline
$\si\subseteq\D$ and $\si$ is a signature & YES,\newline if logic
has PIP &
Proposition \ref{Prop_InseparabilityPreservationUnderForgetting}\\
\hline
$\si\subseteq\D$ and $\bt_1$, $\bt_2$ are semantically inseparable
& \vspace{0.2cm} YES & Proposition
\ref{Prop_ModelInseparabilityPreservationUnderForgetting}\\
\end{tabular}\\
\hline
\vspace{-0.5cm}
Distributivity of forgetting $\si$ over union of $\bt_1$ and
$\bt_2$ &
\begin{tabular}{p{1.3in}|p{3cm}|c}\hline
$\si\cap\D=\varnothing$ & YES & Corollary
\ref{Cor_ForgettingWithSpecComponent}\\ \hline
\vspace{0.4cm} $\si\subseteq\D$ & NO,\newline even if $\bt_1$ and
$\bt_2$ are semantically inseparable & Example
\ref{Ex_ComponentwiseForgetting}\\ \hline
$\bt_1$ and $\bt_2$ are semantically inseparable ``modulo $\si$''
& \vspace{0.1cm} YES & Proposition \ref{Prop_CriterionComponentwiseForgetting}\\
\end{tabular}\\ \hline
\end{tabular} \\

\bigskip

\begin{tabular}[t]{|p{0.74in}|c|} \hline

\hspace{0.4cm} \textbf{Property} & \hspace{1.1cm}
\textbf{Condition} \hspace{2.5cm} \textbf{Preservation}
\hspace{1.2cm} \textbf{Reference}\\ \hline

\vspace{-0.9cm}
$\D$--inseparability of components of initial theory under
progression &

\begin{tabular}{p{1.5in}|p{3cm}|c}

at least one fluent is present in $\D$ & \vspace{0.06cm} NO &
Example \ref{Ex_LossOfInseparability}\\ \hline

$\D$ is fluent-free and some components of initial theory
split under progression & \vspace{0.2cm} NO & Example \ref{Ex_SplitOfComponent}\\
\hline

\vspace{0.1cm} $\D$ is fluent-free and components of initial
theory do not split
under progression \vspace{0.2cm} & \vspace{0.1cm} YES,\newline modulo the unique name assumption theory & Theorem \ref{Teo_PreservationOfInseparability}\\ \hline

\vspace{0.1cm} $\cal{BAT}$ is local--effect, $\D$ is fluent-free and components of initial
theory do not split
under progression \vspace{0.2cm} & \vspace{0.1cm} YES\newline & Theorem \ref{Teo_PreserveCompLocalEffect}\\
\end{tabular}\\ \hline

\vspace{0.3cm} $\D$--decomposabi\-li\-ty and preservation of signa\-tu\-re
com\-po\-nents of an initial theory un\-der progression wrt action 
term \nolinebreak $\alpha$ &

\begin{tabular}[t]{p{1.524in}|p{3cm}|c}
Unconditionally, in particular for local-effect $\cal{BAT}$s and fluent-free $\Delta$'s & \vspace{0.04cm} NO & Example \ref{Ex_DecompLostUnderProgression}\\
\hline

\vspace{0.3cm} $\cal{BAT}$ is local--effect, $\D$ is fluent-free,
and components of $\DD_{S_0}$ are aligned with components of
$\DD_{ss}$
 & \vspace{0.1cm} YES,\newline modulo common symbols of the components of $\DD_{ss}$ and constants in term $\alpha$ \vspace{0.2cm} &
Theorem \ref{Teo_PreserveCompLocalEffect}\\ \hline

\vspace{0.2cm} if additionally the constants in term $\alpha$ are contained in a  single $\Delta$-decomposition component of $\DD_{S_0}$ \vspace{-0.2cm} &  \vspace{0.2cm} YES & Corollary \ref{Cor_StrongPreservationOfComponents} \\
 &  &  \\   
 &  &  \\   
 \end{tabular}\\
\hline
\end{tabular}


\end{center}

\normalsize

The examples and Lemmas given in the paper demonstrate that the sufficient conditions for invariance of decomposability and inseparability wrt progression in local-effect action theories cannot be relaxed. 
Our research has required new understanding of progression and the related notion 
of forgetting wrt modularity of theories. The new results about forgetting are general and may find applicability outside of reasoning about actions. Given a decomposition
of the initial theory into inseparable components, the rest of the conditions in 
Theorem \ref{Teo_PreserveCompLocalEffect} and 
Corollary \ref{Cor_StrongPreservationOfComponents} are purely syntactical and therefore are easy to check. The important practical observation behind 
these results is that in order to compute the progression of an initial theory wrt 
an action having effects only on fluents from one decomposition component, 
it suffices to compute forgetting only in this component. As illustrated by the running example, non-interacting dynamic systems may share only some common names or static entities, such as location. The fact that the dynamic systems share no fluents can be obscured by the way they are presented, whereas decomposition would make it explicit. We believe that our positive results are applicable to a large and general class of basic action theories. 
The significant contribution of the paper is in exploring the important connections between research on modularity and reasoning about action. The paper starts bridging the gap between these two different research communities in knowledge representation.

  There are several directions where future work may proceed. In this paper,
we concentrate on local-effect action theories only. However, recently 
\cite{DeGiacomoLesperancePatriziKR12} defined a new broad class of action theories 
called \textit{bounded} situation calculus action theories, in which actions may
have non-local, but bounded effects. Moreover, for these action theories, one
can find cases when progression is effectively computable \cite{VassosPatriziIJCAI13}.
Therefore, it is natural to explore when decomposability and inseparability
remain invariant wrt progression in bounded action theories. Additionally, we noted
that there is a realistic case of initial theories, for which the size of a progressed 
theory with local effects does not grow exponentially. The initial theories of
this kind are known as \textit{proper$^+$} theories \cite{LakemeyerLevesque02,ProgressionLocalEffect}. Therefore, it is worth while to develop
computationally tractable techniques for decomposition of \textit{proper$^+$} theories. \bigskip

\section{Acknowledgements}
The first author was supported by the German Research Foundation within the Transregional Collaborative Research Centre SFB/ TRR 62 ``Companion-Technology for Cognitive Technical
Systems'' and by Russian Ministry of Science and Education under the 5-100 Excellence Programme, and Russian Foundation for Basic Research, Project No. 15-07-03410A.

The authors would like to thank the Natural Sciences and Engineering Research Council of Canada and the Dept. of Computer Science of the Ryerson University for providing partial financial support.

\bibliographystyle{plain}
\bibliography{decprog}

\begin{thebibliography}{10}

\bibitem{EyalAmirAAAI2000}
Eyal Amir.
\newblock {(De)}composition of situation calculus theories.
\newblock In Henry~A. Kautz and Bruce~W. Porter, editors, {\em AAAI/IAAI},
  pages 456--463. AAAI Press / The MIT Press, 2000.

\bibitem{EyalAmirKR2002}
Eyal Amir.
\newblock Projection in decomposed situation calculus.
\newblock In Dieter Fensel, Fausto Giunchiglia, Deborah~L. McGuinness, and
  Mary-Anne Williams, editors, {\em KR}, pages 315--326. Morgan Kaufmann, 2002.

\bibitem{AmirAIJ2005}
Eyal Amir and Sheila~A. McIlraith.
\newblock Partition-based logical reasoning for first-order and propositional
  theories.
\newblock {\em Artif. Intell.}, 162(1-2):49--88, 2005.

\bibitem{CookLiuJLC2003}
Stephen~A. Cook and Yongmei Liu.
\newblock A complete axiomatization for blocks world.
\newblock {\em J. Log. Comput.}, 13(4):581--594, 2003.

\bibitem{Craig1957Herbrand-Gentzen}
William Craig.
\newblock Three uses of the {Herbrand-Gentzen} theorem in relating model theory
  and proof theory.
\newblock {\em J. Symb. Log.}, 22(3):269--285, 1957.

\bibitem{Craig2008Synthese}
William Craig.
\newblock The road to two theorems of logic.
\newblock {\em Synthese}, 164(3):333--339, 2008.

\bibitem{DeGiacomoLesperancePatriziKR12}
Giuseppe {De Giacomo}, Yves Lesp{\'e}rance, and Fabio Patrizi.
\newblock Bounded situation calculus action theories and decidable
  verification.
\newblock In Gerhard Brewka, Thomas Eiter, and Sheila~A. McIlraith, editors,
  {\em KR}. AAAI Press, 2012.

\bibitem{DeGiacomoMancini2004}
Giuseppe {De Giacomo} and Toni Mancini.
\newblock Scaling up reasoning about actions using relational database
  technology.
\newblock In Deborah~L. McGuinness and George Ferguson, editors, {\em AAAI},
  pages 245--256. AAAI Press / The MIT Press, 2004.

\bibitem{DisjointANDDecomposition}
Pavel Emelyanov and Denis Ponomaryov.
\newblock Algorithmic issues of and-decomposition of boolean formulas.
\newblock {\em Programming and Computer Software}, 41(3):162--169, 2015.

\bibitem{Ershov1980}
Yu.~L. Ershov.
\newblock {\em Decision problems and constructive models. (in Russian)
  [\foreignlanguage{russian}{Проблемы разрешимости и
  конструктивные модели}]}.
\newblock Nauka, Moscow, 1980.

\bibitem{Friedman}
Harvey Friedman.
\newblock The complexity of explicit definitions.
\newblock {\em Advances in Mathematics}, 20:18--29, 1976.

\bibitem{Ghilardi06conservativeextensions}
Silvio Ghilardi, Carsten Lutz, Frank Wolter, and Michael Zakharyaschev.
\newblock Conservative extensions in modal logics.
\newblock In {\em In Proceedings of AiML-6}, pages 187--207. College
  Publications, 2006.

\bibitem{ModularTheoryAndPractice}
Bernardo~Cuenca Grau, Ian Horrocks, Yevgeny Kazakov, and Ulrike Sattler.
\newblock Modular reuse of ontologies: Theory and practice.
\newblock {\em J. Artif. Int. Res.}, 31(1):273--318, February 2008.

\bibitem{Green69}
C.~Cordell Green.
\newblock Application of theorem proving to problem solving.
\newblock In Donald~E. Walker and Lewis~M. Norton, editors, {\em IJCAI}, pages
  219--240. William Kaufmann, 1969.

\bibitem{GruningerApplOntol2012}
Michael Gr{\"u}ninger, Torsten Hahmann, Ali Hashemi, Darren Ong, and Atalay
  {\"O}zg{\"o}vde.
\newblock Modular first-order ontologies via repositories.
\newblock {\em Applied Ontology}, 7(2):169--209, 2012.

\bibitem{GuSoutchanski2010}
Yilan Gu and Mikhail Soutchanski.
\newblock A description logic based situation calculus.
\newblock {\em Ann. Math. Artif. Intell.}, 58(1-2):3--83, 2010.

\bibitem{DecompOnt}
Boris Konev, Carsten Lutz, Denis Ponomaryov, and Frank Wolter.
\newblock Decomposing description logic ontologies.
\newblock In Fangzhen Lin, Ulrike Sattler, and Miroslaw Truszczynski, editors,
  {\em KR}. AAAI Press, 2010.

\bibitem{FormalPropMod}
Boris Konev, Carsten Lutz, Dirk Walther, and Frank Wolter.
\newblock Formal properties of modularisation.
\newblock In Heiner Stuckenschmidt, Christine Parent, and Stefano Spaccapietra,
  editors, {\em Modular Ontologies}, volume 5445 of {\em Lecture Notes in
  Computer Science}, pages 25--66. Springer, 2009.

\bibitem{ModelThInseparability-PositiveResults}
Boris Konev, Carsten Lutz, Dirk Walther, and Frank Wolter.
\newblock Model-theoretic inseparability and modularity of description logic
  ontologies.
\newblock {\em Artif. Intell.}, 203:66--103, 2013.

\bibitem{Makinson}
George Kourousias and David Makinson.
\newblock Parallel interpolation, splitting, and relevance in belief change.
\newblock {\em J. Symb. Log.}, 72(3):994--1002, 2007.

\bibitem{LakemeyerLevesque02}
Gerhard Lakemeyer and Hector~J. Levesque.
\newblock Evaluation-based reasoning with disjunctive information in
  first-order knowledge bases.
\newblock In {\em Proc. of KR-02}, pages 73--81, 2002.

\bibitem{CogRobKRHandbook}
Hector Levesque and Gerhard Lakemeyer.
\newblock Cognitive robotics (chapter 24).
\newblock In Frank van Harmelen, Vladimir Lifschitz, and Bruce Porter, editors,
  {\em Handbook of Knowledge Representation}, pages 869--886. Elsevier, 2007.

\bibitem{LinAIjournal2001}
Fangzhen Lin.
\newblock On strongest necessary and weakest sufficient conditions.
\newblock {\em Artif. Intell.}, 128(1-2):143--159, 2001.

\bibitem{LinKR04}
Fangzhen Lin.
\newblock Discovering state invariants.
\newblock In Didier Dubois, Christopher~A. Welty, and Mary-Anne Williams,
  editors, {\em KR}, pages 536--544. AAAI Press, 2004.

\bibitem{ForgetIt}
Fangzhen Lin and Ray Reiter.
\newblock Forget it!
\newblock In {\em Proceedings of the AAAI Fall Symposium on Relevance}, pages
  154--159, 1994.

\bibitem{LinReiter1997}
Fangzhen Lin and Raymond Reiter.
\newblock How to progress a database.
\newblock {\em Artificial Intelligence}, 92:131--167, 1997.

\bibitem{ProgressionLocalEffect}
Yongmei Liu and Gerhard Lakemeyer.
\newblock On first-order definability and computability of progression for
  local-effect actions and beyond.
\newblock In Craig Boutilier, editor, {\em IJCAI}, pages 860--866, 2009.

\bibitem{LiuLevesqueIJCAI2005}
Yongmei Liu and Hector~J. Levesque.
\newblock Tractable reasoning with incomplete first-order knowledge in dynamic
  systems with context-dependent actions.
\newblock In Leslie~Pack Kaelbling and Alessandro Saffiotti, editors, {\em
  IJCAI}, pages 522--527. Professional Book Center, 2005.

\bibitem{Lutz07conservativeextensions}
Carsten Lutz, Dirk Walther, and Frank Wolter.
\newblock Conservative extensions in expressive description logics.
\newblock In {\em In Proc. of IJCAI-2007}, pages 453--459. AAAI Press, 2007.

\bibitem{ConsExtALC}
Carsten Lutz, Dirk Walther, and Frank Wolter.
\newblock Conservative extensions in expressive description logics.
\newblock In Manuela~M. Veloso, editor, {\em IJCAI}, pages 453--458, 2007.

\bibitem{MathLogicLifeScience}
Carsten Lutz and Frank Wolter.
\newblock Mathematical logic for life science ontologies.
\newblock In Hiroakira Ono, Makoto Kanazawa, and Ruy J. G.~B. de~Queiroz,
  editors, {\em WoLLIC}, volume 5514 of {\em Lecture Notes in Computer
  Science}, pages 37--47. Springer, 2009.

\bibitem{ConsExtEL}
Carsten Lutz and Frank Wolter.
\newblock Deciding inseparability and conservative extensions in the
  description logic $\cal{EL}$.
\newblock {\em J. Symb. Comput.}, 45(2):194--228, 2010.

\bibitem{McC63}
John McCarthy.
\newblock Situations, actions and causal laws.
\newblock Memo 2, Stanford University, Department of Computer Science, 1963.
\newblock Reprinted in: ``Semantic Information Processing'' (M.Minsky, ed.),
  The MIT Press, Cambridge (MA), 1968, pages 410-417.

\bibitem{MH69}
John McCarthy and Patrick Hayes.
\newblock Some philosophical problems from the standpoint of artificial
  intelligence.
\newblock In B.~Meltzer and D.~Michie, editors, {\em Machine Intelligence},
  volume~4, pages 463--502. Edinburgh University Press, Reprinted in
  ``Formalization of common sense: papers by {J}ohn {M}c{C}arthy'' ({V}.
  {L}ifschitz, ed.), Ablex, Norwood, N.J., 1990, 1969.

\bibitem{DecompFOL}
Andrey Morozov and Denis Ponomaryov.
\newblock On decidability of the decomposability problem for finite theories.
\newblock {\em Siberian Mathematical Journal}, 51(4):667--674, 2010.

\bibitem{Mundici}
Daniele Mundici.
\newblock Complexity of craig's interpolation.
\newblock {\em Fundamenta Informaticae}, 5:261--278, 1982.

\bibitem{PR99}
Fiora Pirri and Ray Reiter.
\newblock Some contributions to the metatheory of the situation calculus.
\newblock {\em Journal of the ACM}, 46(3):325--364, 1999.

\bibitem{DecompLogics}
Denis Ponomaryov.
\newblock On decomposibility in logical calculi.
\newblock {\em Bulletin of the Novosibirsk Computing Center}, 28:111--120,
  2008.

\bibitem{ComplexityFOL}
Denis Ponomaryov.
\newblock The algorithmic complexity of decomposability in fragments of
  first-order logic.
\newblock Manuscript. Abstract to appear in the Bulletin of Symbolic Logic.
  http://persons.iis.nsk.su/files/persons/pages/sigdecomp.pdf, 2014.

\bibitem{PonomaryovSoutchanski2013}
Denis Ponomaryov and Mikhail Soutchanski.
\newblock Progression of decomposed situation calculus theories.
\newblock In Marie desJardins and Michael~L. Littman, editors, {\em AAAI}. AAAI
  Press, 2013.

\bibitem{ConceptInterpolation}
Denis Ponomaryov and Dmitry Vlasov.
\newblock Concept definability and interpolation in enriched models of
  el-tboxes.
\newblock In Thomas Eiter, Birte Glimm, Yevgeny Kazakov, and Markus
  Kr{\"o}tzsch, editors, {\em Description Logics}, volume 1014 of {\em CEUR
  Workshop Proceedings}, pages 898--916. CEUR-WS.org, 2013.

\bibitem{Reiter1993}
Raymond Reiter.
\newblock Proving properties of states in the situation calculus.
\newblock {\em Artif. Intell.}, 64(2):337--351, 1993.

\bibitem{Reiter2001}
Raymond Reiter.
\newblock {\em Knowledge in Action: Logical Foundations for Describing and
  Implementing Dynamical Systems}.
\newblock The MIT Press, 2001.

\bibitem{BethDefinabilityDLs}
Balder ten Cate, Enrico Franconi, and Inan{\c{c}} Seylan.
\newblock Beth definability in expressive description logics.
\newblock {\em J. Artif. Intell. Res. {(JAIR)}}, 48:347--414, 2013.

\bibitem{FONonDefinabilityofProgression}
Stavros Vassos and Hector~J. Levesque.
\newblock On the progression of situation calculus basic action theories:
  Resolving a 10-year-old conjecture.
\newblock In Dieter Fox and Carla~P. Gomes, editors, {\em AAAI}, pages
  1004--1009. AAAI Press, 2008.

\bibitem{VassosPatriziIJCAI13}
Stavros Vassos and Fabio Patrizi.
\newblock A classification of first-order progressable action theories in
  situation calculus.
\newblock In Francesca Rossi, editor, {\em IJCAI}. IJCAI/AAAI, 2013.

\bibitem{Vescovo}
Chiara~Del Vescovo, Bijan Parsia, Ulrike Sattler, and Thomas Schneider.
\newblock The modular structure of an ontology: Atomic decomposition.
\newblock In Toby Walsh, editor, {\em {IJCAI} 2011, Proceedings of the 22nd
  International Joint Conference on Artificial Intelligence, Barcelona,
  Catalonia, Spain, July 16-22, 2011}, pages 2232--2237. {IJCAI/AAAI}, 2011.

\bibitem{YehiaDL2012}
Wael Yehia, Hongkai Liu, Marcel Lippmann, Franz Baader, and Mikhail
  Soutchanski.
\newblock Experimental results on solving the projection problem in action
  formalisms based on description logics.
\newblock In Yevgeny Kazakov, Domenico Lembo, and Frank Wolter, editors, {\em
  Description Logics}, volume 846 of {\em CEUR Workshop Proceedings}.
  CEUR-WS.org, 2012.

\end{thebibliography}

\end{document}